\documentclass{article}

\usepackage{microtype}
\usepackage{graphicx}
\usepackage{subfigure}
\usepackage{booktabs} 

\usepackage{hyperref}

\usepackage[accepted]{icml2020}

\icmltitlerunning{Performance Analysis of Generalized Margin Maximizers}
\usepackage{amsfonts}       
\usepackage{nicefrac}       
\usepackage{microtype}      
\usepackage{amsthm}
\usepackage{mathtools,amssymb,algorithm,caption,cases,bm,float,graphicx,url,color}
\usepackage{sidecap}

\newcommand{\mmu}{\bm \mu}

\newtheorem{lem}{Lemma}
\newtheorem{theorem}{Theorem}
\newtheorem{defn}{Definition}
\newtheorem{rem}{Remark}
\newtheorem{cor}{Corollary}

\newcommand{\w}{\mathbf{w}}
\renewcommand{\v}{\mathbf{v}}
\newcommand{\y}{\mathbf{y}}
\renewcommand{\u}{\mathbf{u}}
\newcommand{\Rad}{\textsc{Rad}}
\newcommand{\Prox}{\text{Prox}}
\renewcommand{\P}{\mathbf{P}}
\newcommand{\x}{\mathbf{x}}
\newcommand{\h}{\mathbf{h}}
\newcommand{\g}{\mathbf{g}}
\newcommand{\q}{\mathbf{q}}
\renewcommand{\H}{\mathbf{H}}
\newcommand{\R}{\mathbb{R}}
\newcommand{\normal}{\mathcal{N}}
\newcommand{\Id}{\mathbf{I}}
\newcommand{\la}{\bm \lambda}
\newcommand{\sgn}{\textsc{Sign}}

\newcommand{\lnorm}[1]{\left\Vert {#1} \right\Vert}

\newcommand{\Expect}{\operatorname{\mathbb{E}}}

\newcommand{\Probe}{\mathbb{P}}
\newcommand{\Prob}[1]{\Probe\left\{ #1 \right\}}

\begin{document}

\twocolumn[
\icmltitle{The Performance Analysis of Generalized Margin \\Maximizer (GMM) on Separable Data}





\icmlcorrespondingauthor{Fariborz Salehi}{fsalehi@caltech.com}
\begin{icmlauthorlist}
\icmlauthor{Fariborz Salehi, Ehsan Abbasi,  Babak Hassibi}{Caltech}
\end{icmlauthorlist}

\icmlaffiliation{Caltech}{Department of Electrical Engineering, California Institute of Technology, Pasadena, California, USA}


\icmlkeywords{Machine Learning, ICML}

\vskip 0.3in
]



\printAffiliationsAndNotice{} 

\begin{abstract}
Logistic models are commonly used for binary classification tasks. The success of such models has often been attributed to their connection to maximum-likelihood estimators. It has been shown that gradient descent algorithm, when applied on the logistic loss, converges to the max-margin classifier (a.k.a. hard-margin SVM). The performance of the max-margin classifier has been recently analyzed in~\cite{montanari2019generalization, deng2019model}. Inspired by these results, in this paper, we present and study a more general setting, where the underlying parameters of the logistic model possess certain structures (sparse, block-sparse, low-rank, etc.) and introduce a more general framework (which is referred to as “Generalized Margin Maximizer”, GMM). While classical max-margin classifiers minimize the $2$-norm of the parameter vector subject to linearly separating the data, GMM minimizes any arbitrary convex function of the parameter vector. We provide a precise analysis of the performance of GMM via the solution of a system of nonlinear equations. We also provide a detailed study for three special cases: ($1$) $\ell_2$-GMM that is the max-margin classifier, ($2$) $\ell_1$-GMM which encourages sparsity, and ($3$) $\ell_{\infty}$-GMM which is often used when the parameter vector has binary entries. Our theoretical results are validated by extensive simulation results across a range of parameter values, problem instances, and model structures.
\end{abstract}
\section{Introduction}
\label{sec:intro}
Machine learning models have been very successful in many applications,  ranging from spam detection, face and pattern recognition, to the analysis of genome sequencing and financial markets. However, despite this indisputable success, our knowledge on why the various machine learning methods exhibit the performances they do is still at a very early stage. To make this gap between the theory and the practice narrower, researchers have recently begun to revisit simple machine learning models with the hope that understanding their performance will lead the way to understanding the performance of more complex machine learning methods.\\
More specifically, studies on the performance of diffrent classifiers for binary classification dates back to the seminal work of Vapnik in the 1980's~\cite{vapnik1982estimation}. In an effort to find the ''optimal'' hyperplane that separates the data, he presented an upper bound on the test error which is inversely proportional to the margin (minimum distance of the datapoints to the separating hyperplane), and concluded that the max-margin classifier is indeed the desired classifier.  It has also been observed that to construct such optimal hyperplanes one only has to take into
acconnt a small amount of the training data, the so-called support vectors~\cite{cortes1995support}. \\
In this paper, we challenge the conventional wisdom by showing that when the underlying parameter has certain structure one can come up with classifiers that outperform the max-margin classifier. We introduce the {\bf{G}}eneralized {\bf{M}}argin {\bf{M}}aximizer ({\bf{GMM}}) which takes into account the structure of the underlying parameter as well as the minimimum distance of the datapoints to the separating hyperplane. We provide sharp asymptotic results on various performance measures (such as the generalization error) of GMM and show that an appropriate choice of the potential function can in fact improve the resulting estimator. 
\subsection{Prior work}
There have been many recent attempts to understand the generalization behavior of simple machine learning models~\cite{bartlett2019benign, mei2019generalization, xu2019many, belkin2018overfitting, hastie2019surprises}. Most of these studies focus on the least-squares/ridge regression, where the loss function is the squarred $\ell_2$-norm, and derive sharp asymptotics on the performance of the estimator. In particular, in~\cite{hastie2019surprises, kini2020analytic} the authors have shown that the minimum-norm least square solution demonstrates the so-called "double-descent" behavior~\cite{belkin2019reconciling}. \\
A more recent line of research studies the generalization performance of gradient descent (GD) for binary classification. It has been shown~\cite{soudry2018implicit}) that for a separable dataset, GD (when applied on the logistic loss) converges in direction to the max-margin classifier (a.k.a. hard-margin SVM). The performance of max-margin classifier has been recently analyzed in two independent works~\cite{montanari2019generalization, deng2019model}.
\subsection{Summary of contributions}
Inspired by the recent results in understanding the performance of the max-margin classifier, in this paper we introduce and study a more general framework. We assume the underlying parameters possess certain structure (e.g. sparse) and introduce the generalized margin maximizer (GMM) as the solution of a convex optimization problem whose objective function encourages the structure.\\
We analyze the performance of GMM in the high-dimensional regime where both the number of parameters, $p$, and the number of samples $n$ grows, and analyze the asymptotic performance as a function of the overparameterization ratio $\delta:=\frac{p}{n}>0$. First, we provide the phase transition condition for the separability of data (i.e., derive the exact value of $\delta^*$ such that the data is separable for all $\delta>\delta^*$\footnote{Concurrent to the submission of this paper, a similar phase transition has been demonstrated in~\cite{kini2020analytic} for a somewhat different model.}.) Consequently, we analyze the performance in the interpolating regime ($\delta>\delta^*$). To the best of our knowledge, this is the first theoretical result that provides sharp asymptotics on the performance of GMM classifiers on separable data.
For our analysis, we exploit the  {\bf{C}}onvex {\bf{G}}aussian {\bf{M}}in-max {\bf{T}}heorem ({\bf{CGMT}})~\cite{stojnic2013framework, thrampoulidis2015regularized} which is a strengthened version of a classical Gaussian comparison inequality due to Gordon~\cite{gordon1985some}.  This framework
replaces the original optimization with another optimization problem that has a similar performance, yet is  much simpler to analyze as it becomes nearly separable. Previously, the CGMT has been successfully applied to derive the precise performance in a number of applications such as regularized M-estimators~\cite{thrampoulidis2018precise}, analysis of the generalized lasso~\cite{miolane2018distribution, thrampoulidis2015regularized}, data detection in massive MIMO~\cite{abbasi2019performance, atitallah2017ber, thrampoulidis2019simple},  and PhaseMax in phase retrieval~\cite{dhifallah2018phase, salehi2018learning, salehi2018precise}.\\ More recently, this framework has been employed in a series of works by multiple groups of researchers to characterize the performance of the logistic loss minimizer in binary classification~\cite{salehi2019impact, taheri2019sharp}. Furthermore, in an analogous avenue of research, the CGMT framework has been utilized to study the generalization behavior of the gradient descent algorithm in the interpolating regime, where there exists a (nonempty) set of parameters that perfectly fit the training data~\cite{montanari2019generalization, deng2019model}.\\
The organization of the paper is as follows: In Section~\ref{sec:prelim} we mathematically introduce the problem and the notations used in the paper. Section~\ref{sec:main} contains the main results of the paper where we first provide the asymptotic phase transition on the separability of the data, and then in our main theorem, we present the precise performance analysis of GMM, which then be used to compute the generalization error. We investigate our theoretical findings for three specific cases of potential functions in Section~\ref{sec:structured}. Numerical simulations for the genralization error of the GMM classifiers are presented in Section~\ref{sec:num_sim}.
We should note that most technical derivations of the results presented in the paper are deferred to the Appendix. 
\section{Preliminaries}
\label{sec:prelim}
\subsection{Notations}
Here, we gather the basic notations that are used throughout the paper. $X\sim p_{X}$ denotes that the random variable $X$ has a density $p_X$. $\normal{(\bm\mu,\bm\Sigma)}$ denotes the multivariate Gaussian distribution with mean $\bm \mu$, and covariance $\bm \Sigma$, and $\Rad(p)$, for $p\in [0,1]$, is the symmetric bernouli random variable which takes the value $+1$ with ptobability $p$, and $-1$ with probability $1-p$. $\overset{D}\rightarrow$, and $\overset{P}\rightarrow$ represent convergence in distribution and in probability, respectively. Bold lower letters are reserved for vectors, and upper letters are for matrices. $\mathbf 1_d$, and $\Id_d$ respectively represent the all-one vector and the identity matrix in dimension $d$. For a vector $\v$, $v_i$ denotes its $i$-th entry, and $\lnorm{\v}_p$ (for
$p \geq 1$), is its $\ell_p$ norm, where we remove the subscript when $p = 2$. For a scalar $t\in \R$, $(t)_+ = \max(t,0)$ denotes its positive part, and $\sgn(t)$ indicates its sign. \\
A function $f:\R^d\rightarrow \R$ is called (invariantly) separable, when for all $\w \in \R^d$, $f(\w) = \sum_{i=1}^d\tilde f(w_i)$, for a real-valued function $\tilde f$. For a function $\Phi:\mathbb R^d\rightarrow \mathbb R$, the Moreau envelope associated with $\Phi(\cdot)$ is defined as,
\begin{equation}
\label{eq:Moreau}
M_{\Phi}(\mathbf v, t) = \min_{\mathbf x\in \mathbb R^d}~~\frac{1}{2t}||\mathbf v-\mathbf x||^2+\Phi(\mathbf x)~,
\end{equation}
and the proximal operator is the solution to this optimization, i.e.,
\begin{equation}
\text{Prox}_{t\Phi(\cdot)}(\mathbf v) = \arg\min_{\mathbf x\in \mathbb R^d}~~\frac{1}{2t}||\mathbf v-\mathbf x||^2+\Phi(\mathbf x)~.
\end{equation}
Finally, the function $\Phi(\cdot)$ is said to be locally-Lipschitz if for any $M>0$, there exists a constant $L_M$, such that,
\begin{equation}
    \forall \u,\v \in [-M,+M]^d,~~|\Phi(\u) - \Phi(\v)|\leq L_{M}\lnorm{\u-\v}.
\end{equation}
\subsection{Mathematical setup}
\label{sec:setup}
We consider the problem of binary classification, having a set of training data, $\mathcal D=\{(\x_i,y_i)\}_{i=1}^{n}$, where each of the sample points consists of a $p$-dimensional feature vector, $\x_i$, and a binary label, $y_i\in\{\pm 1\}$. We assume that the dataset $\mathcal D$ is generated from a logistic-type model with the underlying parameter $\w^\star\in \R^p$. This means that
\begin{equation}
    y_i\sim \Rad(\rho(\x_i^T\w^\star))~,\quad i=1,\dots,n~,
\end{equation}
where $\rho: \R\rightarrow [0,1]$ is a non-decreasing function and is often referred to as the link function. A commonly-used instance of the link function is the standard logistic function defined as $\rho(t):=\frac{1}{1+e^{-t}}$. \\
When $n/p$ is sufficiently large, i.e., when we have access to a sufficiently large number of samples, the maximum-likelihood estimator( $\hat{\mathbf w}_{ML}$) is well-defined. In such settings, the MLE is often the estimator of choice due to its desirable properties in the classical statistics. Sur and Cand{\`{e}}s~\cite{sur2018modern} have recently studied the performance of the MLE in logistic regression in the high-dimensional regime, where the number of observations and parameters are comparable, and show, among other things, that the maximum likelihood estimator is biased. Their results have been extended to regularized logistic regression~\cite{salehi2019impact}, assuming some prior knowledge on the structure of the data. In particular, it has been observed that, when the regularization parameter is tuned properly, the regularized logistic regression can outperform the MLE.\\
Inspired by the recent results on analyzing the generalization error of machine learning models, in this paper, we study the generalization error of binary classification, in a regime of parameters known as the interpolating regime. Here, the assumption is that there exists a parameter vector that can perfectly fit (interpolate) the data, i.e.,
\begin{equation}
    \exists \w_0~\text{ s.t. }~\sgn{(\w_0^T\x_i)} = y_i, ~~\text{for }i=1,2,\ldots,n.
\end{equation}
Let $\mathcal W$ denote the set of all the parameters that interpolate the data. 
\begin{equation}
    \mathcal W = \{\w\in\R^p:\sgn(\w^T\x_i)=y_i~,~~\text{for }1\leq i\leq n.\}.
\end{equation}
It has been observed that in many machine learning tasks, the iterative solvers that minimize the loss function often converge to one of the points in the set $\mathcal W$ (the training error converges to zero).  Therefore,  one can (qualitatively) pose the following important (yet still mysterious) question:

~~~~\fbox{\parbox{0.42\textwidth}{%
Which point(s) in $\mathcal W$ is (are)  ''better'' estimator(s)\\
of the actual parameter, $\w^\star$?
}}

In an attempt to find an answer to this question, we focus on the simple (yet fundamental) model of binary clasification. We assume that the underlying parameter, $\w^\star$ possesses certain structure (sparse, low-rank, block-sparse, etc.), and consider a locally-Lipschitz and convex function $\psi:\R^p\rightarrow \R$ which encourages this structure. We introduce the \emph{Generalized Margin Maximizer} (GMM) as the solution to the following optimization:
\begin{equation}
    \label{eq:opt_main}
    \begin{aligned}
    &&&\min_{\w \in \R^p}~~~\psi(\w)\\
&&&~~~\text{ s.t. }y_i(\x_i^T\w)\geq 1, ~~\text{for }1\leq i\leq n.
    \end{aligned}
\end{equation}
It is worth noting that the condition on the separability of the dataset is crucial for the optimization program~\eqref{eq:opt_main} to have a feasible point.
\begin{rem}
It can be shown that when $\psi(\cdot)$ is absolutely scalable\footnote{A function $f:\R^d\rightarrow \R$ is absolutely scalable when, $$\forall \v\in \R^d, \forall \alpha\in \R,~~f(\alpha\v) = |\alpha|f(\v).$$ ~~~~~All $\ell_p$ norms, for example, are absolutely scalable.}, the GMM can be found by solving the following equivalent optimization program,
\begin{equation}
    \label{eq:gen_margin_equivalent}
    \max_{\w \in \R^d}\frac{\psi(\w)}{\underset{1\leq i \leq n}\min y_i(\x_i^T\w)} = \max_{\w \in \R^d} \frac{\lnorm{\w}}{\underset{1\leq i \leq n}\min y_i(\x_i^T\w)}\times \frac{\psi(\w)}{\lnorm{\w}}.
\end{equation}
The first multiplicative term on the right indicates the margin associated with the separator $\w$, and the second term, $\frac{\psi(\w)}{\lnorm{\w}}$ takes into account the structure of the model. Hence, we refer to the objective function in  the optimization~\eqref{eq:gen_margin_equivalent} as the generalized margin, and the solution to this optimization is called the generalized margin maximizer (GMM).
\end{rem}
In this paper, we study
the linear asymptotic regime in which the problem dimensions $p,~n$ grow to infinity at a proportional rate, $\delta:=\frac{p}{n}>0$. Our main result characterizes the performance of  the solution of~\eqref{eq:opt_main}, $\hat \w$, in terms of the ratio, $\delta$, and
the signal strength, $\kappa := \frac{\lnorm{\w^\star}}{\sqrt{p}}$. We assume that the datapoints, $\{\x_i\}_{i=1}^n$, are drawn independently from the Gaussian distribution. Our main result characterizes the
performance of the resulting estimator through the solution of a system of five nonlinear equations
with five unknowns. In particular, as an application of our main result, we can accurately predict the generalization error of the resulting estimator.

\section{Main Results}
\label{sec:main}
In this section, we present the main results of the paper, that is the characterization of the performance of the generalized margin maximizers. Our results are represented in terms of a summary functional, $c_t(\cdot,\cdot)$, which incorporates the informaiton about the underlying model. 
\begin{defn}
\label{def:c_kappa}
    For the parameter $t>0$, the function  $c_t:\R\times\R_+\rightarrow \R_+$ is defined as,
    \begin{equation}
        \label{eq:def_c_kappa}
        c_t(s,r) = \Expect{\big[(1-tsZ_1Y-rZ_2)_{+}^{2}\big]},
    \end{equation}
    where $Z_1,Z_2\overset{\text{i.i.d.}}\sim \normal(0,1)$, and $Y\sim \Rad(\rho(t Z_1))$.
\end{defn}
\subsection{Asymptotic phase transition}\label{sec:phase_transition}
Here, we provide the necessary and sufficient condition for the separability of the data. 
\begin{theorem}[Phase transition]\label{thm:phase_transition}
Consider the generalized max margin optimization defined in Section~\ref{sec:setup}. As $n,p\rightarrow \infty$ at a fixed overparameterization ratio $\delta:=\frac{p}{n}\in(0,\infty)$, this optimization program (almost surely) has a solution (or  equivalenty, the set $\mathcal W$ is nonempty) if and only if,
\begin{equation}
    \label{eq:phase_transition}
    \delta>\delta^*=\delta^*(\kappa) := \underset{s,r\geq 0}\inf~\frac{c_{\kappa}(s,r)}{r^2}~.
\end{equation}
\end{theorem} 
\begin{rem}
Theorem~\ref{thm:phase_transition} indicates the necessary and sufficient condition for the existense of GMM. It is worth mentioning that this condition, which is simply the condition on separability of the dataset $\mathcal D$, does not depend on the choice of the potential function $\psi(\cdot)$.
\end{rem}
\begin{figure}[t]
\vskip 0.2in
\begin{center}
\centerline{\includegraphics[width=260pt]{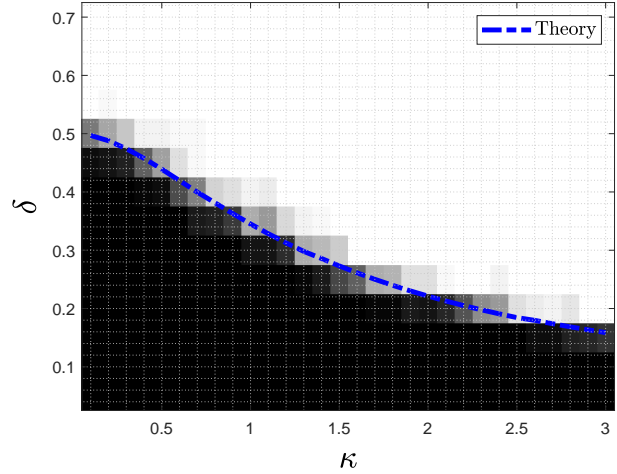}}
\caption{The phase transition, $\delta^*$, for the separability of the dataset, where the feature vector, $\x_i$ is drawn from the Gaussian distribution, $\normal(\mathbf 0, \frac{1}{p}\Id_p)$, and the labels are $y_i\sim \Rad\big(\rho(\x_i^T\w^\star)\big)$, for $\rho(z) = \frac{e^t}{e^t + e^{-t}}$. The empirical result is the average over $20$ trials with $p=150$, and the theoretical results are from Theorem~\ref{thm:phase_transition}.}
\label{fig:phase}
\end{center}
\vskip -0.2in
\end{figure}
\begin{rem}
The phase transition~\eqref{eq:phase_transition}, is valid for any link function $\rho(\cdot)$. This generalizes the former results in~\cite{candes2018phase}. Note that the summary functional, $c_{\kappa}(\cdot,\cdot)$, contains the choice of the link function and can be computed numerically. 
\end{rem}
The following lemma explains the behavior of $\delta^*$ as $\kappa$ varies.
\begin{lem}\label{lem:one}
$\delta^*$ is a decreasing function of $\kappa$, with $\delta^*(0) = \frac{1}{2}$ and $\lim_{\kappa\rightarrow +\infty}\delta^*(\kappa) =0.$
\end{lem}
The result of Lemma \ref{lem:one} can be intuitively verified. Recall that $\kappa=\frac{\|\mathbf w^\star\|}{\sqrt p}$ and $y_i\sim \Rad(\rho(\x_i^T\w^\star))$. Therefore, $\kappa\rightarrow\infty$ translates to having $y_i=\sgn(\x_i^T\w^\star)$. In this case our training data is always separable for any number of observations $n$. Besides, the case of $\kappa=0$ corresponds to having random labels assigned to feature vectors $\mathbf x_i$. \cite{cover1965geometrical} showed that in this case, as $p\rightarrow \infty$, $\delta>0.5$ is the necessary and sufficient condition for the separability of the dataset.\\
Figure~\ref{fig:phase} provides a comparison between the theoretical result in Theorem~\ref{thm:phase_transition}, and the empirical results derived from numerical simulations for $p=150$ and $20$ trials. As seen in this plot, the theory matches well with the empirical simulations. 
\subsection{A nonlinear system of equations}
\label{sec:nonlin_system}
Our main result in Section~\ref{sec:main_thm} precisely characterizes the performance of GMM in terms of a system of $5$ nonlinear equations with $5$ unknowns, $(\alpha,\sigma,\beta,\gamma,\tau),$ defined as follows,
\begin{equation}
\label{eq:nonlinsys}
\begin{cases}
\frac{1}{p}\Expect\big[ {\w^\star}^T\mathbf P\big] = \alpha\kappa^2,\\
\frac{1}{p}\Expect\big[ \h^T\mathbf P\big] = \sqrt{\frac{c_{\kappa}(\alpha,\sigma)}{\delta}},\\
\frac{1}{p}\Expect\lnorm{\mathbf P}^2 = \alpha^2\kappa^2 + \sigma^2,\\
\frac{\partial c_{\kappa}(\alpha,\sigma)}{\partial \alpha} = \frac{2\kappa^2\gamma}{\beta}\sqrt{c_{\kappa}(\alpha,\sigma)},\\
\frac{\partial c_{\kappa}(\alpha,\sigma)}{\partial \sigma} = \frac{2\sqrt{c_{\kappa}(\alpha,\sigma)}}{\beta \tau},
\end{cases}
\end{equation}
where $\mathbf P$ is defined as,
\begin{equation}
    \mathbf P = \Prox_{\sigma\tau\psi(\cdot)}\big((\alpha-\sigma\tau\gamma)\w^\star+\beta \sigma\tau\sqrt\delta \h\big)
\end{equation}
\begin{rem}
The first three equations in the nonlinear system~\eqref{eq:nonlinsys} capture the role of the potential function, via its proximal operator. When $\psi(\cdot)$ is separable, these functions can further be reduced to the proximal operator of a real-valued function. For instance, when $\psi(\cdot)=\lnorm{\cdot}_1$, the proximal operator is simply equivalent to applying the well known shrinkage (defined as $\eta(x,t)=\frac{x}{|x|}(|x|-t)_{+}$) on each entry. For more information on the proximal operators, please refer to~\cite{parikh2014proximal}.
\end{rem}
\subsection{Asymptotic performance of GMM}
\label{sec:main_thm}
We are now ready to present the main result of the paper. Theorem~\ref{thm:main} chracaterizes the asymptotic behavior of GMM, that is the solution to the optimization program~\eqref{eq:opt_main}. It connects the performance of GMM to the solution of the nonlinear system of equations~\eqref{eq:nonlinsys}, and {\underline{informally}} states that,
\begin{equation}
    \hat \w \overset{D}\rightarrow \Gamma(\w^\star, \h),\text{ as }p\rightarrow \infty,
\end{equation}
where $\h \in \R^p$ has standard normal entries, and $\Gamma:\R^p\times \R^p\rightarrow \R^p$ is defined as,
\begin{equation}
\label{eq:Gamma_def}
    \Gamma(\v_1,\v_2) = \Prox_{\bar\sigma\bar\tau\psi(\cdot)}\big((\bar\alpha-\bar\sigma\bar\tau\bar\gamma)\v_1+\bar\beta \bar\sigma\bar\tau\sqrt\delta \v_2\big),
\end{equation}
where $(\bar\alpha,\bar\sigma, \bar \beta, \bar \gamma, \bar \tau)$ is the solution to the nonlinear system~\eqref{eq:nonlinsys}.
\begin{theorem}
\label{thm:main}
Let $\hat \w$ be the solution of the GMM optimization~\eqref{eq:opt_main}, where for $i=1,2,\ldots,n$, $\x_i$ has the multivariate Gaussian distribution $\normal(\mathbf 0, \frac{1}{p}\Id_p)$, and $y_i \sim \Rad(\rho(\x_i^T\w^\star))$, and $\w^\star$ is drawn from a distribution $\Pi$ with $\kappa =\frac {\lnorm{\w^\star}}{\sqrt{p}}$. As $n,p\rightarrow \infty$ at a fixed overparameterization ratio $\delta=\frac{p}{n}>\delta^*(\kappa)$, the nonlinear system~\eqref{eq:nonlinsys} has a unique solution $(\bar\alpha,\bar\sigma, \bar \beta, \bar \gamma, \bar \tau)$. Furthermore, for any locally-Lipschitz function $F:\R^p\times \R^p\rightarrow \R$, we have,
\begin{equation}
\label{eq:main_perf_measure}
    F(\hat \w, \w^\star)\overset{P}\rightarrow \Expect[F(\Gamma(\w,\h),\w)],
\end{equation}
where $\h\in \R^p$ has standard normal entries, $\w\sim \Pi$ is independent of $\h$, and the function $\Gamma(\cdot,\cdot)$ is defined in~\eqref{eq:Gamma_def}.
\end{theorem}
The detailed proof of this result is deferred to Appendix~\ref{sec:app_pf_main_thm}. In short, we introduce dual variables and write down the Lagrangian which contains a bilinear form with respect to a matrix with i.i.d. Gaussian entries. Exploiting the CGMT framework, we then analyze the nearly-separable auxiliary optimization to find its optimal value, and show that the nonlinear system~\eqref{eq:nonlinsys} corresponds to its optimality condition. 
\begin{rem}
The result in Theorem~\ref{thm:main} is stated for a general locally-Lipschitz function $F(\cdot, \cdot)$. To evaluate a specific performance measure, one can appeal to this theorem with an appropriate choice of $F$. As an example, the function $F(\u,\v)=\frac{1}{p}\lnorm{\u-\v}^2$ gives the mean-squarred error (MSE).
\end{rem}
\vspace{-0.9em}
\subsection{Generalization error}
\vspace{-0.6em}
Theorem~\ref{thm:main} can be utilized to derive useful information on the performance of the classifier. In fact, using this theorem one can show that the parameters $\bar \alpha$, and $\bar \sigma$ respectively correspond to the correlation (to the underlying parameter)  and the mean-squared error of the resulting estimator. \\
An important measure of performance is the generalization error, which indicates the success of the trained model on unseen data. Here, we compute the generalization error of the GMM classifier. We do so, by appealing to the result of Theorem~\ref{thm:main}.
\begin{defn}
The generalization error for a binary classifier with parameter $\hat \w$ is defined as,
\begin{equation}
    {GE}_{\hat \w} = \mathbb P_{\x}\{{\sgn(\x^T\hat\w)\neq \sgn(\x^T\w^\star)}\},
\end{equation}
where the probability is computed with respect to the distribution of the test data.
\end{defn}
It can be shown that when the distribution of the test data is rotationally invariant (e.g., Gaussian, uniform dist. on the unit-sphere), GE only depends on the angle between $\hat \w$ and $\w^\star$.  The following lemma provides sharp asymptotics on the generalization error of the GMM classifier.
\begin{lem}[Generalization Error]
\label{lem:gen_error}
Let $\hat \w$ be the GMM classifier defined in Section~\ref{sec:setup}. Assume $\delta>\delta^*$, and the (test) data is distributed according to the multivariate Gaussian distribution $\normal(\mathbf 0, \frac{1}{p}\Id_p)$. Then, as $p\rightarrow \infty$, we have,
\begin{equation}
    GE_{\hat \w} \overset{P}\rightarrow \frac{1}{\pi}\text{acos}(\frac{\kappa\bar\alpha}{\sqrt{\kappa^2{\bar{\alpha}}^2+{\bar \sigma}^2}}),
\end{equation}
where $\bar \alpha$ and $\bar \sigma$ are derived by solving the nonlinear system~\eqref{eq:nonlinsys}.
\end{lem}
\begin{proof}
We first note that when the data is normally distributed, the generalization error for $\hat \w$ is defined as,
\begin{equation}
    \label{eq:lem_2_1}
    GE_{\hat\w} = \frac{1}{\pi}\text{acos}(\frac{{\hat \w}^T\w^\star}{\lnorm{\w^\star}\lnorm{\hat \w}}).
\end{equation} 
We  appeal to the result of Theorem~\ref{thm:main} with two different functions. Using $F_1(\u,\v) = \frac{1}{p}\v^T\u$ in~\eqref{eq:main_perf_measure} will give,
\begin{equation}
    \label{eq:lem_2_2}
    \frac{1}{p}\hat\w^T\w^\star\overset{P}\rightarrow\frac{1}{p}\Expect\big[ {\w^\star}^T\Prox_{\bar\sigma\bar\tau\psi(\cdot)}\big((\bar\alpha-\bar\sigma\bar\tau\bar\gamma)\w^\star+\bar\beta \bar\sigma\bar\tau\sqrt\delta \h\big)\big].
\end{equation}
Since $(\bar\alpha, \bar \sigma, \bar \beta, \bar \gamma, \bar \tau)$ is the solution to the nonlinear system, we can replace the expectation from the first equation in~\eqref{eq:nonlinsys},which gives the following,
\begin{equation}
    \label{eq:lem_2_3}
    \frac{1}{p}\hat\w^T\w^\star\overset{P} \rightarrow \kappa^2\bar\alpha. 
\end{equation}
Similarly, using the result of Theorem~\ref{thm:main} for the measure function $F_2(\u, \v)=\frac{1}{p}\lnorm{\u}^2$, along with the third equation in~\eqref{eq:nonlinsys} gives, 
\begin{equation}
    \label{eq:lem_2_4}
    \frac{1}{\sqrt{p}}\lnorm{\hat\w} \overset{P}\rightarrow  \sqrt{\kappa^2{\bar\alpha}^2+{\bar\sigma}^2}~.
\end{equation}
The proof is the consequence of~\eqref{eq:lem_2_1},~\eqref{eq:lem_2_3}, and~\eqref{eq:lem_2_4}, along with the continuity of the  function $\text{acos}(\cdot)$. 
\end{proof}

\vspace{-1.0em}
\section{GMM for Various Structures}
\label{sec:structured}
 As explained earlier, the potential function $\psi(\cdot)$ is chosen to encourage the structure of the underlying parameter. In this section, we investigate the performance of the GMM classifier for some common structures and the corresponding choices of the potential function.
 \vspace{-0.5em}
 \subsection{Max-margin classifier ($\ell_2$-GMM)}
 The $\ell_2$-norm regularization is commonly used in machine learning applications to stabilize the model. Here, we study the performance of the GMM classifier when $\psi(\cdot)=\frac{1}{2}\lnorm{\cdot}_2^2$, i.e., the solution to the following optimization program,
 \begin{equation}
    \label{eq:l2_GMM}
    \begin{aligned}
    &&&\min_{\w \in \R^p}~~~\frac{1}{2}\lnorm{\w}_2^2\\
&&&~~~\text{ s.t. }y_i(\x_i^T\w)\geq 1, ~~\text{for }1\leq i\leq n.
    \end{aligned}
\end{equation}
The optimization program~\eqref{eq:l2_GMM} is called the hard-margin SVM and the corresponding solution is the max-margin classifier, as it maximizes the minimum distance (margin) of the datapoints from the separating hyperplane. As mentioned earlier in Section~\ref{sec:intro}, the conventional justification for using such a classifier is that the risk of a classifier is inversely proportional to its margin. The performance of  $\ell_2$-GMM~\eqref{eq:l2_GMM}, has been earlier analyzed in~\cite{deng2019model} and~\cite{montanari2019generalization}. The form we present below in \eqref{eq:nonlinsys_l2}, differes in appearance to the results of \citep{deng2019model}, but can be shown to be equivalent.
\\
When $\psi(\cdot)=\frac{1}{2}\lnorm{\cdot}_2^2$, the proximal operator has the following closed-form,
\begin{equation}
    \label{eq:prox_l2}
    \text{Prox}_{\frac{t}{2}\lnorm{\cdot}^2}(\u) = \frac{1}{1+t}\u. 
\end{equation}
By replacing the proximal operator in the nonlinear system~\eqref{eq:nonlinsys}, we can explicitly find two of the variables ($\beta$, and $\gamma$) and reduce it to the following system of three nonlinear equations in three unknowns,
\begin{equation}
\label{eq:nonlinsys_l2}
    \begin{cases}
    \begin{aligned}
    &&\sqrt{c_{\kappa}(\alpha,\sigma)} &= \sigma\sqrt{\delta},\\
        &&\frac{\partial c_{\kappa}(\alpha,\sigma)}{\partial \alpha} &= \frac{-2\kappa^2\alpha\tau\sigma\delta}{1+\sigma\tau},\\
    &&\frac{\partial c_{\kappa}(\alpha,\sigma)}{\partial \sigma} &= \frac{2\sigma\delta}{1+\sigma\tau}.
    \end{aligned}
    \end{cases}
\end{equation}
\vspace{-0.6em}
\subsection{Sparse classifier ($\ell_1$-GMM)}
\vspace{-0.6em}
In today's machine learning applications, typically the number of available features, $p$, is overwhelmingly large. To reduce the risk of overfitting in such settings, feature selection methods are often performed to exclude irrelevent variables from the model~\cite{james2013introduction}.  Adding an $\ell_1$ penalty is the most popular approach for feature selection.\\
As a natural consequence of our main result in Theorem~\ref{thm:main}, here we analyze the asymptotic performance of GMM when the potential function is the $\ell_1$ norm, and evaluate its success on the unseen data (i.e., the test error) when the underlying parameter, $\w^\star$, is sparse.
 \begin{equation}
    \label{eq:l1_GMM}
    \begin{aligned}
    &&&\min_{\w \in \R^p}~~~\lnorm{\w}_1\\
&&&~~~\text{ s.t. }y_i(\x_i^T\w)\geq 1, ~~\text{for }1\leq i\leq n.
    \end{aligned}
\end{equation}
In this case, the proximal operator of the potential function ($\lnorm{\cdot}_1$) is basically equivalent to applying the soft-thresholding operator, on each entry, i.e.,
\begin{equation}
    \label{prox_l1}
    \text{Prox}_{t\lnorm{\cdot}_1}(\u) = \eta(\u,t),
\end{equation}
where $\eta(x,t) := \frac{x}{|x|}(|x|-t)_+$ is the soft-thresholding operator. Here, for a sparsity factor $s\in (0,1]$, we assume the entries of $\w^\star$ are sampled i.i.d. from the following distribution,
\begin{equation}
\label{eq:sparse_dist}
    \Pi_s(w) = (1-s)\cdot \delta_0(w) + s\cdot \big(\frac{\phi(\frac{w}{\frac{\kappa}{\sqrt s}})}{\frac{\kappa}{\sqrt s}}\big),
\end{equation}
where $\delta_0(\cdot)$ is the Dirac delta function, and $\phi(t):=\frac{e^{-\frac{t^2}{2}}}{\sqrt{2\pi}}$ is the density of the standard normal random variable. This means that each of the entries of $\w^\star$ are zero with probability $1-s$, and the nonzero entries have independent Gaussian distribution with variance $\frac{\kappa^2}{s}$. Having this assumption we can further simplify the first three equations in the nonlinear system~\eqref{eq:nonlinsys}, and present them in terms of q-functions. To streamline our representation, we introduce the following proxies,
\begin{equation}
\label{eq:proxies}
    t_1 = \frac{\sigma\tau}{\sqrt{\frac{\kappa^2}{s}(\alpha-\sigma\tau\gamma)^2 + \beta^2\sigma^2\tau^2\delta}}~,~~ t_2 =\frac{1}{\beta\sqrt{\delta}}.
\end{equation}
We also define the function $\chi:\mathbb R\rightarrow \R_+$ as,
\begin{equation}
\begin{aligned}
&&\chi(t) &= \Expect{\big[(Z-t)_+^2\big]}~,~~~~Z\sim \mathcal N(0,1)\\
&&&= Q(t)(1+t^2) -t\phi(t),
\end{aligned}
\end{equation}
Where $Q(t):=\int_{t}^{\infty}\phi(x)dx$ denotes the tail distribution of standard normal random variable. We are now able to simplify the first three equations in~\eqref{eq:nonlinsys} and derive the following nonlinear system,
\begin{equation}
\label{eq:nonlinsys_l1}
    \begin{cases}
    Q(t_1) = \frac{\alpha}{2(\alpha-\sigma\tau\gamma)},\\
    s\cdot Q(t_1) + (1-s)\cdot Q(t_2) = \frac{\sqrt{c_{\kappa}(\alpha,\sigma)}}{2\beta\sigma\tau\delta},\\
    \frac{s}{t_1^2}\cdot\chi(t_1)+\frac{(1-s)}{t_2^2}\cdot\chi(t_2) = \frac{\kappa^2\alpha^2}{2\sigma^2\tau^2}+ \frac{1}{2\tau^2},\\
    \frac{\partial c_{\kappa}(\alpha,\sigma)}{\partial \alpha} = \frac{2\kappa^2\gamma}{\beta}\sqrt{c_{\kappa}(\alpha,\sigma)},\\
    \frac{\partial c_{\kappa}(\alpha,\sigma)}{\partial \sigma} = \frac{2\sqrt{c_{\kappa}(\alpha,\sigma)}}{\beta \tau}.    
    \end{cases}
\end{equation}
The nonlinear system~\eqref{eq:nonlinsys_l1} can be solved via numerical methods. For our numerical simulations in Section~\ref{sec:num_sim} we exploit accelerated fixed-point methods to solve the nonlinear system. Using the the result of Lemma~\ref{lem:gen_error}, we can compute the generalization error. \\
Another important measure in this setting (when $\w^\star$ is sparse) is the probability of error in support recovery. Let $\Omega\subseteq [p]$ denote the support of $\w^\star$ (i.e. $\Omega=\{j:\w^\star_j\neq 0\}$.) For a pre-defined threshold $\epsilon$, we form the following estimate of the support,
\begin{equation}
    \hat\Omega_{\epsilon} = \{j:1\leq j\leq p, |{\hat \w}_j|>\epsilon\}.
\end{equation}
The following lemma establishes the success in the support recovery:
\begin{lem}[Support Recovery]
\label{lem:supp_rec}
For a sparsity factor $s\in(0,1]$, let the entries of $\w^\star$ have distribution $\Pi_s$ defined in~\eqref{eq:sparse_dist}, and $\hat \w$ be the solution to the optimization~\eqref{eq:l1_GMM}. Then, as $p\rightarrow \infty$, we have,
\begin{equation}
    \label{eq:support_recovery}
    \begin{aligned}
    &&\lim_{\epsilon\downarrow 0} P_1(\epsilon) &:= \Prob{j\notin \hat\Omega_{\epsilon}|j\in \Omega}\overset{P}\rightarrow 1-2Q(\bar{t}_1)\\
    &&\lim_{\epsilon\downarrow 0} P_2(\epsilon) &:= \Prob{j\in \hat\Omega_{\epsilon}|j\notin \Omega}\overset{P}\rightarrow 2Q(\bar{t}_2)~,
    \end{aligned}
\end{equation}
where $\bar{t}_1$ and $\bar{t}_2$ are defined as in~\eqref{eq:proxies}, with variables derived from solving the nonlinear system~\eqref{eq:nonlinsys_l1}.
\end{lem}
\subsection{Binary classifier ($\ell_{\infty}$-GMM)}
As the last example of structured classifiers, here we study the case where $\w^\star\in\{\pm\}^p$. To encourage this structure, the potential function is chosen to be the $\ell_{\infty}$ norm. In linear regression, $\lnorm{\cdot}_{\infty}$ is used to recover the binary signals, i.e., when $\w^\star \in \{\pm 1\}^p$~\cite{chandrasekaran2012convex}. This problem arises in integer programming and has some connections to the Knapsack problem~\cite{mangasarian2011probability}. Here, we consider analyzing the performance of the solution of the following optimization program,
 \begin{equation}
    \label{eq:l_inf_GMM}
    \begin{aligned}
    &&&\min_{\w \in \R^p}~~~\lnorm{\w}_{\infty}\\
&&&~~~\text{ s.t. }y_i(\x_i^T\w)\geq 1, ~~\text{for }1\leq i\leq n.
    \end{aligned}
\end{equation}
It can be shown that the proximal operator of the $\ell_{\infty}$-norm can be derived by projecting the points onto the $\ell_1$-ball. We use this connection to present the proximal operator in this case in terms of the soft-thresholding operator $\eta(\cdot, \cdot)$.\\
For a vector $\w$ whose entries are drawn independently from a distribution $\Pi$, we can present the following formula for the proximal operator:
\begin{equation}
    \text{Prox}_{tp\lnorm{\cdot}_{\infty}}(\w) = \w - \text{Prox}_{\lambda\lnorm{\cdot}_1}(\w),
\end{equation}
where $\lambda:=\lambda(t)$ is the smallest nonnegative number that satisfies,
\begin{equation}
    \Expect{\big[|\eta(W,\lambda)|\big]} = \Expect{\big[(|W|-\lambda)_+\big]}\leq t.
\end{equation}
Here, the expectation is with respect to $W\sim \Pi$. Note that $\lambda$ is a non-increasing function of $t$, and $\lambda=0$ whenever $t\ge \Expect{|W|}$. \\
Similar to the case of $\ell_1$-GMM, here we can use the closed-form of the proximal operator to simplify the first three equations in the nonlinear system~\eqref{eq:nonlinsys}. For our numerical simulations in the next section, we have done the computations for three different distributions: (1) The i.i.d. Gaussian distribution, (2) the sparse distribution defined in~\eqref{eq:sparse_dist}, and (3) the uniform binary distribution, $\Pi = \text{Unif}\big(\{\pm 1\}^p\big)$.  We postpone the details of the theoretical derivations for this part to Appendix~\ref{sec:app_bin_classifier}.  

\begin{figure}[t]
\vskip 0.2in
\begin{center}
\centerline{\includegraphics[width=260pt]{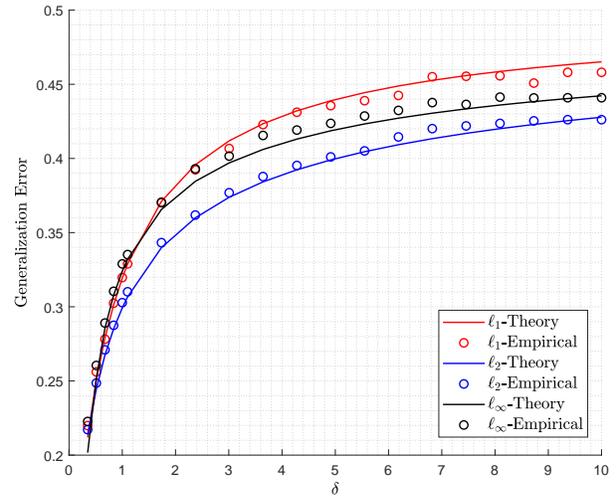}}
\caption{Generalization error of the general max margin classifier under three penalty functions, $\ell_1$ norm with the red line ($\ell_1$-GMM), $\ell_2$ norm with the blue line ($\ell_2$-GMM), and $\ell_\infty$ norm with the black line ($\ell_\infty$-GMM). \textbf{In this figure, the entries of $\mathbf w^\star$ are drawn independently from $\mathcal N(0,\kappa^2)$ Gaussian distribution}. Solid lines correspond to the theoretical results derived from Theorem \ref{thm:main}, while the circles are the result of empirical simulations. For the numerical simulations, the result is the average over 100 independent trials with $p=200$ and $\kappa=2$.}
\label{gaussian_w}
\end{center}
\vskip -0.2in
\end{figure}
\vspace{-0.8em}
\section{Numerical Simulations}\label{sec:num_sim}
\vspace{-0.5em}

In this section, we investigate the validity of our theoretical results with multiple numerical simulations applied to the three different cases of GMM classifiers elaborated in Section~\ref{sec:num_sim}. For each of the three potentials discussed in the paper (i.e., $\ell_1$, $\ell_2$, and $\ell_{\infty}$ norms) we perform numerical simulations for three different models on the distribution of $\w^\star$. In other words, we change the distribution of the entries of $\w^\star$ and evaluate the performance of the aforementioned classifiers on each model. As will observed in our numerical simulations, the appropriate choice of the potential function in the GMM optimization~\eqref{eq:opt_main} has an impact on the generlization error of the resulting classifier. The three different distribution that we choose for the underlying parameter are as follows:
\begin{figure}[t]
\begin{center}
\centerline{\includegraphics[width=260pt]{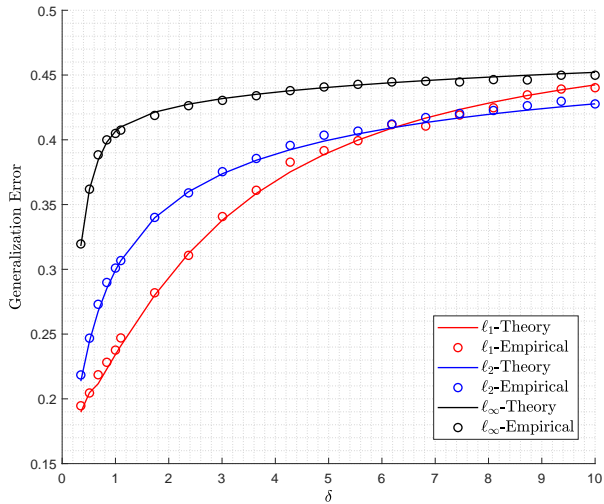}}
\caption{Generalization error of the general max margin classifier under three penalty functions, $\ell_1$ norm with the red line ($\ell_1$-GMM), $\ell_2$ norm with the blue line ($\ell_2$-GMM), and $\ell_\infty$ norm with the black line ($\ell_\infty$-GMM). \textbf{In this figure, the underlying vector $\mathbf w^\star$ is $s$-sparse, where the non-zero entries are drawn independently from $\mathcal N(0,\kappa^2/s)$ Gaussian distribution}. Solid lines correspond to the theoretical results derived from Theorem \ref{thm:main}, and the circles are the result of empirical simulations. For the numerical simulations, the result is computed by taking the average over 100 independent trials with $p=200$, $s=.1$ and $\kappa=2$.}
\label{sparse_w}
\end{center}
\vskip -0.3in
\end{figure}

{\bf{ Gaussian}}: in the first model, we assume that the entries of $\w^\star$ are drawn from a zero-mean Gaussian distribution, $\normal(0, \kappa^2)$. In this model, the direction of $\w^\star$ (which indicates the separating hyperplane) is distributed uniformly on the unit sphere. Figure~\ref{gaussian_w} gives the generalization error when $\w^\star$ has Gaussian distribution. The solid lines show the theoretical results derived from Theorem~\ref{thm:main} and Lemma~\ref{lem:gen_error}. The circles depict empirical results that are computed by taking the average over $100$ trials with $p=200$ and $\kappa=2$. Although our theory provides the generalization error in the asymptotic regime, it appropriately matches the result of empirical simulations in our simulations in finite dimensions. It can be observed in this figure that the max-margin classifier ($\ell_2$-GMM) outperforms the other two classifiers. We should also note that as the overparameterization ratio, $\delta$, grows the generalization error increases which indicates that the estimator is not reliable for large values of $\delta$.\\
{\bf{ Sparse}}: here, we assume that the entries of $\w^\star$ are drawn from the sparse distribution represented in~\eqref{eq:sparse_dist}, i.e., each entry is nonzero with probability $s$, and the nonzero entries have i.i.d. Gaussian distribution with appropriately-defined variance. Figure~\ref{sparse_w} demonstrates the result of the numerical simulations for this model for the three different classifiers of interest. The empirical result is the average over $100$ trials with $p=200$, $s=0.1$, and $\kappa=2$. Similar to the previous case, the empirical results match the theory. Also, it can be observed that the $\ell_1$-GMM outperforms the two other classifiers in the regime of $\delta$ that the classifiers performs well (i.e. $\delta \precapprox 6$.) Similarly, we can observe that for large values of $\delta$ all the classifiers perform poorly.\\
\begin{figure}[t]
\vskip 0.2in
\begin{center}
\centerline{\includegraphics[width=270pt]{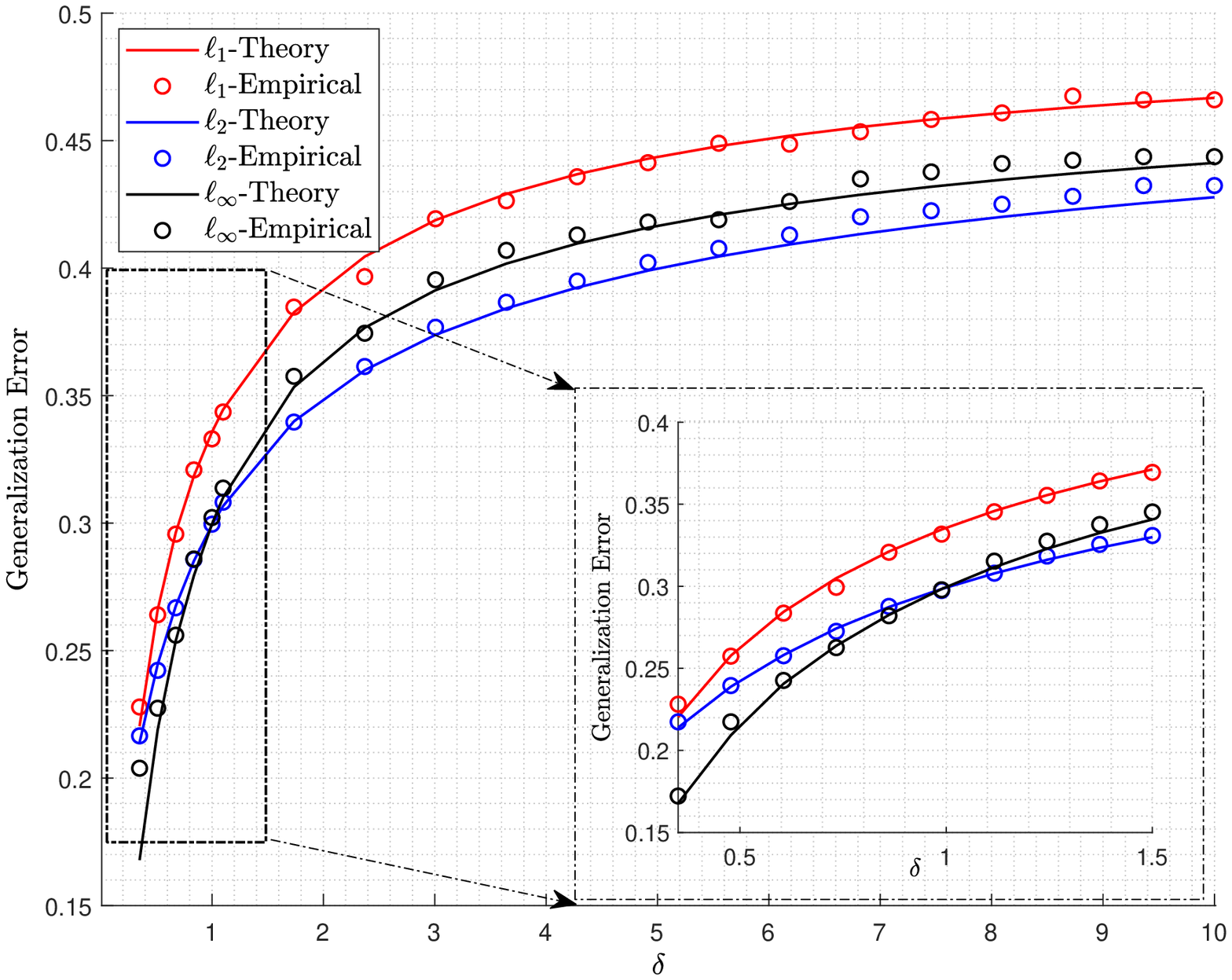}}
\caption{Generalization error of the general max margin classifier under three penalty functions, $\ell_1$ norm with the red line ($\ell_1$-GMM), $\ell_2$ norm with the blue line ($\ell_2$-GMM), and $\ell_\infty$ norm with the black line ($\ell_\infty$-GMM). \textbf{In this figure, the entries of $\mathbf w^\star$ are drawn independently from $\kappa*\Rad(0.5)$ Rademacher distribution}. Solid lines correspond to the theoretical results derived from Theorem \ref{thm:main}, and the circles are the result of empirical simulations. For the numerical simulations, the result is the average over 100 independend trials with $p=200$ and $\kappa=2$.}
\label{fig:binary_w}
\end{center}
\vskip -0.3in
\end{figure}
{\bf{ Binary}}: in this model the entries of $\w^\star$ are independently drawn from $\{+\kappa, -\kappa\}$, i.e., $\w^\star$ is uniformly chosen on the discrete set $\{\pm\kappa\}^p$. Figure~\ref{fig:binary_w} shows the result of numerical simulations under this model. Similar to previous cases the empirical results ($\kappa=2,~p=200$) match the theory. Also, the $\ell_{\infty}$-GMM classifier outperforms the other two classifiers for $\delta<1$ (which corresponds to the underparameterized setting). However, the max-margin classifier performs better for larger values of $\delta$. 

\section{Conclusion and Future Directions}
In this paper, we introduced the generalized margin maximizers (GMM) as a way to extend the max-margin classifiers to structured models. To this end, we  proposed an optimization program whose objective function is a convex potential function $\psi(\cdot)$ that encourages the underlying structure, and the constraints are similar to the max-margin classifier (hard-margin SVM). Our main result in Theorem~\ref{thm:main} provides the asymptotic behavior of GMM classifier for any locally-Lipschitz performance measure via solving a system of nonlinear equations. We utilize this result to characterize the generalization error in the asymptotic regime.\\ 
We examined our theoretical findings on three specific choices of the potential function, $\ell_1$, $\ell_2$, and $\ell_{\infty}$ norms. We simplified the nonlinear systems for each of these functions and validated our theoretical results in numerical simulations by doing simulations on three different structures on the underlying parameter, $\w^*$. The numerical simulations indicates that for sparse signals, $\ell_1$-GMM outperforms the max-margin classifier ($\ell_2$-GMM). We also observed that for binary signals, when $\delta<1$, the $\ell_{\infty}$-GMM outperforms the two other classifiers.\\
In future works, we would like to extend our theory to predict some common phenomena (e.g. the double descent) for GMM. Also,  another avenue of pursuit is to  design iterative optimizaiton algorithms that would converge to the GMM classifier.

\newpage
\bibliography{library}
\bibliographystyle{icml2020}
\newpage

\appendix
\onecolumn
\section*{Appendix}
\section{Proof of Theorem~\ref{thm:main}}\label{sec:app_pf_main_thm}
Here, we present the proof of the main result of the paper. Recall that the generalized margin maximizer is defined as the soluiton to the following optimization program,
\begin{equation}
        \label{eq:app_main_opt}
    \begin{aligned}
    &&&\min_{\w \in \R^p}~~~\psi(\w)\\
&&&~~~\text{ s.t. }y_i(\x_i^T\w)\geq 1, ~~\text{for }1\leq i\leq n.
    \end{aligned}
\end{equation}
Theorem~\ref{thm:main} provides a precise characterization on the performance of this optimization program in the asymptotic regime, where $n,p\rightarrow \infty$ at a fixed ratio $\delta:=n/p$. We assume the datapoints are drawn independently from the multivariate gaussian distribution, i.e., $\x_i\overset{\text{i.i.d.}}\sim\normal(\mathbf 0, \frac{1}{p}\Id_{p})$.

For our analysis we utilize the CGMT framework (see Appendix~\ref{sec:CGMT}), which will provide us with a nearly-separable optimization program that has the same performance as~\eqref{eq:app_main_opt}. To simplify the presentation, we are breaking down the proof into the following three steps:
\begin{enumerate}
    \item Finding the auxiliary optimization: By introducing dual variables, we present the optimization~\eqref{eq:app_main_opt} as a bilinear form with respect to a Gaussian matrix. Consequently, we use the result of Theorem~\ref{thm:CGMT} and Corollary~\ref{cor:CGMT} to find the auxiliary optimization.
    \item Analyzing the auxiliary optimization: The first step provides a nearly-separable optimization. The purpose of this step is to simplify this optimization and present it in terms of an optimization program with respect to scalar variables.
    \item Optimality condition of the auxiliary optimization: By taking the derivatives with respect to various scalars, we present the first-order optimality condition on the solution of the (simplified) auxiliary optimization. Further simplification gives the nonlinear system~\eqref{eq:nonlinsys}.  
\end{enumerate}
We explain each of the three steps in more details in the following subsections.
\subsection{Finding the auxiliary optimization}
\label{sec:find_aux}
The following lemma presents the auxiliary optimization associated with the GMM optimization~\eqref{eq:app_main_opt}.
\begin{lem}
\label{lem:define_aux}
Let $\hat \w$ be the solution to the optimization~\eqref{eq:app_main_opt}. Consider the following optimization:
\begin{equation}
\begin{aligned}
    \label{eq:app_4}
    &&&\min_{\substack{\alpha\in \R\\\tilde\w \in \R^p\\\tilde \w\perp \w^\star}} ~\frac{1}{p}\psi(\alpha\w^\star+\tilde \w)\\ 
    &&&~\quad\qquad\text{s.t.}~~~\frac{1}{p}(\h^T\tilde \w)^2\geq n\cdot c_{\kappa}\big(\alpha, \frac{\lnorm{\tilde \w}}{\sqrt{p}}\big),
\end{aligned}
\end{equation}
where $\h\in \R^p$ has i.i.d. standard normal entries. Assume $(\bar \alpha, \bar{\tilde \w})\in \mathbb R\times \R^p$ be the solution to this optimization program. Then, as $p\rightarrow \infty$, we have:
\begin{equation}
    \label{eq:app_5}
    \lnorm{\hat \w - (\bar \alpha \w^\star+ \bar{\tilde \w})} \overset{P} \longrightarrow 0.
\end{equation}
\end{lem}
\begin{proof}
In order to apply the CGMT, we need to have a $\min$-$\max$ optimization.  Introducing the Lagrange variable, $\la:=[\lambda_1,\lambda_2,\ldots,\lambda_n]^T\in \R_+^n$, we can rewrite  the optimization program as follows,
\begin{equation}
    \label{eq:app_6}
    \min_{\w \in \R^p}\max_{\la\in \R_+^n}~ \frac{1}{p}\psi(\w) + \frac{1}{n} \sum_{i=1}^{n}\lambda_i\big(1-y_i(\x_i^T\w)\big).
\end{equation}
Note that the scaling has been performed in such a way that all the terms in the objective be of constant order. We define the matrix $\H\in \R^{n\times p}$ as,
\begin{equation}
    \label{eq:app_7}
    \H:=-\sqrt{p}\cdot\begin{bmatrix}-\x_1^T-\\-\x_2^T-\\ \vdots \\ -\x_n^T-\end{bmatrix}.
\end{equation}
Based on the assumption on the distribution of datapoints, this matrix has i.i.d. standrad normal $\normal(0,1)$ entries. To ease the notation, we also define a new variable $\bar \la= \la \odot \y$ (i.e., ${\bar \lambda}_i = \lambda_i y_i$) and reformulate the optimization~\eqref{eq:app_6} as,
\begin{equation}
    \label{eq:app_8} 
    \min_{\w \in \R^p}\max_{\substack{\bar \la\in \R^n\\\bar\lambda_i y_i\geq 0}} ~\frac{1}{p}\psi(\w) + \frac{1}{n} \bar\la^T\y+\frac{1}{n\sqrt{p}}{\bar \la}^T\H\w.
\end{equation}
We proceed by analyzing the optimization program~\eqref{eq:app_8}. In order to apply the CGMT, we need an additive bilinear form that is {\underline{ statistically independent}} of other functions that appear in the objective. Note that the label vector $\y\in \{\pm\}^n$ is a random variables that depends on $\H\w^\star$, as $\y = \Rad\big(\rho(-\frac{1}{\sqrt{p}}\H\w^\star)\big)$. Therefore, to remove this independence between $\y$ and the bilinear form, we use the projection onto $\w^\star$. Let $\P$ be the matrix of orthogonal projection onto $\text{span}(\w^\star)$, i.e., $\P = \frac{1}{\lnorm{\w^\star}^2}\w^\star{\w^\star}^T$, and $\P^\perp$ be its orthogonal completent, $\P^\perp = \Id_p - \P$. We use these projection matrices to decompose the Gaussian matrix $\H$ as $\H=\H_1+\H_2$ with $\H_1:= \H\times \P$, and $\H_2:=\H\times \P^\perp$. This gives the following equivalent optimization,
\begin{equation}
    \label{eq:app_9}
    \min_{\w\in \R^p}\max_{\substack{\bar \la\in \R^n\\\bar\lambda_i y_i\geq 0}} ~\frac{1}{p}\psi(\w) + \frac{1}{n} \bar\la^T\y+\frac{1}{n\sqrt{p}}{\bar \la}^T\H_1\w+\frac{1}{n\sqrt{p}}{\bar \la}^T\H_2{\w}.
    \tag{PO}
\end{equation}
It is worth noting that the projections of a Gaussian matrix (or vector) onto orthogonal subspaces are statistically independent. Also, the label vector $\y$ would be independent of $\H_2$ since,
\begin{equation}
    \label{eq:app_10}
    \y = \Rad\big(\rho(-\frac{1}{\sqrt{p}}\H\w^\star)\big) = \Rad\big(\rho(-\frac{1}{\sqrt{p}}\H\P\w^\star)\big) = \Rad\big(\rho(-\frac{1}{\sqrt{p}}\H_1\w^\star)\big),
\end{equation}
where we used $\P\w^\star = \w^\star$. Therefore, all the additive terms in the objective function of~\eqref{eq:app_9} except the last one are independent of $\mathbf H_2$. Also, the objective function is convex with respect to $\w$ and concave(linear) with respect to $\bar \la$. In order to apply the CGMT framework (Theorem~\ref{thm:CGMT}), we only need an extra condition which is restricting the feasibility sets of $\w$, and $\bar \la$ to be compact and convex. We can introduce some artificial convex and bounded  sets $\mathcal S_\w$,  and $\mathcal S_{\bar \la}$, and perform the optimization over these sets. Note that these sets can be chosen large enough such that they do not affect the optimization itself. For simplicity, in our arguments here we ignore the condition on the compactness of the fesible sets and apply the CGMT whenever the variables are defined on a convex domain.

The optimization program~\eqref{eq:app_9} is suitable to be analyzed via the CGMT as the conditions are all satisfied. Having identified~\eqref{eq:app_9} as the primary optimization, it is straightforward to write its corresponding auxiliary optimization (AO) [as in~\eqref{eq:POAO}, c.f. Appendix~\ref{sec:CGMT}]. The Auxiliary Optimization (AO) can be written as follows,
\begin{equation}
    \label{eq:app_11}
    \min_{\w\in \R^p}\max_{\substack{\bar \la\in \R^n\\\bar\lambda_i y_i\geq 0}} ~\frac{1}{p}\psi(\w) + \frac{1}{n} \bar\la^T\y+\frac{1}{n\sqrt{p}}{\bar \la}^T\H_1\w+\frac{1}{n\sqrt{p}}\big(\lnorm{\bar \la}\h^T\P^\perp \w+{\bar \la}^T\g\lnorm{\P^\perp \w}\big),
    \tag{AO}
\end{equation}
where $\h \in \R^p$ and $\g \in \R^n$ have i.i.d. standard normal entries.
Next, we decompose $\w$ as $\w:=\P\w + \P^\perp \w = \alpha\w^\star + \tilde\w$, where $\alpha\in \R$, and $\tilde \w\in \R^p$ is such that $\tilde \w \perp \w^\star$. We also define the vector $\q:= -\frac{1}{\kappa\sqrt{p}}\H\w^\star$. Note that since $\lnorm{\w} = \kappa\sqrt{p}$, the entries of $\q$ have standrad normal distribution. Therefore, we have the following equivalent optimization,
\begin{equation}
    \label{eq:app_12}
    \min_{\substack{\alpha\in \R\\\tilde\w \in \R^p\\\tilde \w\perp \w^\star}}\max_{\substack{\bar \la\in \R^n\\\bar\lambda_i y_i\geq 0}} ~\frac{1}{p}\psi(\alpha\w^\star+\tilde \w) + \frac{1}{n} \bar\la^T\y-\frac{\alpha\kappa}{n}{\bar \la}^T\q+\frac{1}{n\sqrt{p}}\big(\lnorm{\bar\la} \h^T\tilde \w + {\bar\la}^T\g\lnorm{\tilde \w}\big)~,
\end{equation}
Proceeding onwards, we solve the inner optimization ($\max_{\bar \la}$) with respect to the direction of $\bar \la$. We have:
\begin{align}
\label{eq:app_13}
  &&\max_{\substack{\bar \la\in \R^n\\\bar\lambda_i y_i\geq 0}}\frac{1}{n} \bar\la^T\y-\frac{\alpha\kappa}{n}{\bar \la}^T\q+\frac{1}{n\sqrt{p}}\big(\lnorm{\bar\la} \h^T\tilde \w+{\bar\la}^T\g\lnorm{\tilde \w}\big)=\max_{\substack{\bar \la\in \R^n\\\bar\lambda_i y_i\geq 0}} &  ~\frac{\lnorm{\bar \la}}{\sqrt n} (\frac{1}{\sqrt{np}}\h^T\tilde \w + \frac{1}{\sqrt n}\lnorm{\mmu})\\
  &&& \text{s.t. }~\mu_i = \big(1-\alpha\kappa q_iy_i + \frac{\lnorm{\tilde\w}}{\sqrt p} g_iy_i\big)_+,~~\text{for }1\leq i\leq n.\nonumber
\end{align}
Recall that the function $c_{\kappa}:\R\times \R_{+}\rightarrow \R$ is defined (c.f. Definition~\ref{def:c_kappa}) as follows:
\begin{equation}
    \label{def:cd_app}
    c_{\kappa}(t_1,t_2) = \Expect{\big(1-\kappa t_1  Z_1Y + t_2Z_2\big)_+^2}, 
\end{equation}
where $Z_1,Z_2$ are independent standard normal random variables, and $Y \sim \Rad\big(\rho(\kappa Z_1)\big)$. Therefore, we have $\Expect{\mu_i^2} = c_{\kappa}(\alpha, \frac{\lnorm{\tilde \w}}{\sqrt p})$, and using the SLLN, as $p,n \rightarrow \infty$ we can  replace $\lnorm{\mmu}$ with $\sqrt{n\cdot c_{\kappa}(\alpha, \frac{\lnorm{\tilde \w}}{\sqrt p})}$ due to the almost sure convergence. Introducing the positive variable $\beta = \frac{\lnorm{\bar \la}}{\sqrt n}$, we have the following reformulation of the auxiliary optimization,
\begin{equation}
    \label{eq:app_14}
    \min_{\substack{\alpha\in \R\\\tilde\w \in \R^p\\\tilde \w\perp \w^\star}}\max_{\beta\geq 0} ~\frac{1}{p}\psi(\alpha\w^\star+\tilde \w)+\frac{\beta}{\sqrt{np}}\h^T\tilde \w + \beta \cdot \sqrt{c_{\kappa}\big(\alpha,\frac{\lnorm{\tilde \w}}{\sqrt p}\big)}~~.
\end{equation}
We can write the inner maximization (with respect to $\beta$) as a constraint for the optimization, which gives the same formulation as~\eqref{eq:app_4}, i.e.,
\begin{equation}
\begin{aligned}
    \label{eq:app_15}
    &&&\min_{\substack{\alpha\in \R\\\tilde\w \in \R^p\\\tilde \w\perp \w^\star}} ~\frac{1}{p}\psi(\alpha\w^\star+\tilde \w)\\ 
    &&&~\quad\qquad\text{s.t.}~~~\frac{1}{p}(\h^T\tilde \w)^2\geq n\cdot c_{\kappa}\big(\alpha, \frac{\lnorm{\tilde \w}}{\sqrt{p}}\big),
\end{aligned}
\end{equation}
Using the result of Corollary~\ref{cor:CGMT}, we have that when the solution of the primary optimization converges as the problem dimensions grow ($p\rightarrow \infty$), the solution of the auxiliary optimization converges to the same set (point). This concludes the proof.
\end{proof}

\subsection{Analyzing the auxiliary optimization}\label{sec:analyze_aux}
In this section we analyze the performance of the refined version of the auxiliary optimization in~\eqref{eq:app_14}. Although this optimization program is (nearly) separable, it is still a high-dimensional optimization. Ideally, one would like to simplify this optimzation to obtain another optimization program in lower dimensions (with respect to a few scalar variables)  where the performance can be numerically computed. To do so, in this section we exploit some tools from convex analysis along with some tricks from calculus to further simplify the optimization program~\eqref{eq:app_14}.

The goal is to express the final result in terms of the \emph{expected Moreau envelope}  of the regularization function. To better understand the behavior of the solution in~\eqref{eq:app_14} we first introduce some new variables, $\u, \v\in \R^p$, and $\gamma \in \R$ and write the optimization as follows,
\begin{equation}
    \label{eq:app_16}
    \min_{\substack{\alpha\in \R\\\u,\tilde\w \in \R^p\\}}\max_{\substack{\beta\geq 0,\gamma\\\v\in \R^p}} ~\frac{1}{p}\psi(\u)+\frac{\beta}{\sqrt{np}}\h^T\tilde \w + \beta \cdot \sqrt{c_{\kappa}\big(\alpha,\frac{\lnorm{\tilde \w}}{\sqrt p}\big)}+\frac{1}{p}\v^T\big(\u - \alpha\w^\star-\tilde\w\big)+\frac{\gamma}{p}\w^{\star T}  \tilde \w.
\end{equation}
The variable $\u$  has been introduced to detach the impact of $s$ and $\tilde \w$ from $\psi(\cdot)$. The variables $\v$ and $\gamma$ are Lagrange dual variables to remove the constraints from the optimization.  We shall emphasize again that the normalization has been performed to ensure that all the terms in the objective are of constant order. Next, we would like to solve the minimization with respect to $\tilde \w$. 

Before continuing our analysis, we need to discuss an important point that would help us in the remaining of this section. It will be observed that in order to simplify the optimization, we would like to flip the orders of $\min$ and $\max$ in the (AO) optimization.  Since the objective function in the  auxiliary optimization is not convex-concave we cannot appeal to the Sion's min-max theorem in order to flip $\min$ and $\max$. However,  it has been shown (see Appendix~A in~\cite{thrampoulidis2018precise}) that flipping the orders of $\min$ and $\max$ in the (AO) is allowed {\underline{in the asymptotic setting}}. This is mainly due to the fact that the original (PO) optimization was convex-concave with respect to its variables, and as the CGMT suggests (AO) and (PO) are tightly related in the asymptotic setting; hence, flipping the order of optimizations in (AO) is justified whenever such a flipping is allowed in the (PO).  We appeal to this result to flip the orders of $\min$ and $\max$ when needed.

Next, we solve the optimization with respect to the direction of $\tilde \w$. Defining $\sigma:=\lnorm{\tilde\w}/{\sqrt{p}}$ and solving the optimization with respect to the direction of $\tilde \w$ leads to,
\begin{equation}
    \label{eq:app_17}
    \min_{\substack{\sigma\geq 0,\alpha\\\u \in \R^p\\}}~\max_{\substack{\beta\geq 0,\gamma\\\v\in \R^p}} ~\frac{1}{p}\psi(\u) + \beta \cdot \sqrt{c_{\kappa}\big(\alpha,\sigma\big)}+\frac{1}{p}\v^T\big(\u - \alpha\w^\star) -\sigma\cdot\lnorm{\frac{\beta}{\sqrt n} \h-\frac{1}{\sqrt p}\v+\frac{\gamma}{\sqrt p} \w^\star}.
\end{equation}
Consequently, we are considering the maximization with respect to the vector variable $\v\in \R^p$. As seen in~\eqref{eq:app_17} this variable appears in the last two additive terms in the objective function. To find the optimal value for $\v$, we introduce a new scalar variable $\tau>0$~\footnote{The square-root trick: This is adopted from~\cite{thrampoulidis2018precise}, where it was proposed in the analysis of the auxiliary optimization in regularized M-estimators, and the idea is to use the following equivalence (which is derived immediately from AM-GM inequality):$$\sqrt{x} = \min_{\tau>0}~ \frac{1}{2\tau}x + \frac{\tau}{2}~,~~\forall x>0.$$}, which simplifies the optimization by changing $\lnorm{\cdot}$ to $\lnorm{\cdot}^2$. The new optimization would be,
\begin{equation}
    \label{eq:app_18}
    \min_{\substack{\sigma\geq 0,\alpha\\\u \in \R^p\\}}~\max_{\substack{\beta\geq 0,\tau> 0,\gamma\\\v\in \R^p}} ~\frac{1}{p}\psi(\u) + \beta \cdot \sqrt{c_{\kappa}\big(\alpha,\sigma\big)}+\frac{1}{p}\v^T\big(\u - \alpha\w^\star) - \frac{\sigma\tau}{2} -\frac{\sigma}{2\tau}\cdot\lnorm{\frac{\beta}{\sqrt n} \h-\frac{1}{\sqrt p}\v+\frac{\gamma}{\sqrt p} \w^\star}^2.
\end{equation}
It can be easily check that the optimization programs~\eqref{eq:app_17} and~\eqref{eq:app_18} are equivalent by simply solving the inner optimization with respect to the variable $\tau$. We are now ready to solve the optimization with respect to $\v$. To do so, we continue by making a completion of squares as follows,
\begin{align}
    \label{eq:app_19}
    &&\frac{1}{p}\v^T\big(\u - \alpha\w^\star) -\frac{\sigma}{2\tau}\cdot\lnorm{\frac{\beta}{\sqrt n} \h-\frac{1}{\sqrt p}\v+\frac{\gamma}{\sqrt p} \w^\star}^2 &= -\frac{\sigma}{2\tau}\cdot\lnorm{\frac{\beta}{\sqrt n} \h-\frac{1}{\sqrt p}\v+\frac{\gamma}{\sqrt p} \w^\star+\frac{\tau}{\sigma\sqrt p}\u-\frac{\alpha\tau}{\sigma\sqrt p}\w^\star}^2\nonumber\\
    &&& ~~~+ \frac{\tau}{2\sigma p}\lnorm{\u-\alpha\w^\star}^2+\frac{\beta}{\sqrt{np}}\u^T\h+\frac{\gamma}{p}\u^T\w^\star-\frac{\alpha\beta\sqrt{\delta}}{p}\h^T\w^\star-\alpha\gamma\kappa^2,\nonumber\\
    &&{\Large{\boxed{p,n\rightarrow +\infty~}}}\qquad\qquad\qquad&= -\frac{\sigma}{2\tau}\cdot\lnorm{\frac{\beta}{\sqrt n} \h-\frac{1}{\sqrt p}\v+\frac{\gamma}{\sqrt p} \w^\star+\frac{\tau}{\sigma\sqrt p}\u-\frac{\alpha\tau}{\sigma\sqrt p}\w^\star}^2\nonumber\\
    &&&~~~+ \frac{\tau}{2\sigma p}\lnorm{\u+\frac{\beta \sigma\sqrt{\delta}}{\tau}~\h+(\frac{\sigma\gamma}{\tau}-\alpha) \w^\star}^2-\frac{\sigma}{2\tau}(\delta\beta^2+\gamma^2\kappa^2),
\end{align}
Where we exploit the fact that, as $p\rightarrow \infty$, we can replace $\frac{1}{p}\lnorm{\w^\star}^2$ and $\frac{1}{p}\lnorm{\h}^2$ with $\kappa^2$ and $1$, respectively. Furthermore, we omit the term $\frac{1}{p} \h^T\w^\star = \mathcal O(\frac{1}{\sqrt p})$ as its negligible compare to other terms in the optimization (which are of constant $\mathcal O(1)$ orders.) Using the above  completion-of-squares $\v$ is now appearing in only one quadratic term in~\eqref{eq:app_19}. Hence, To maximize the objective, $\v$ chooses itself in such a way that it makes the quadratic term equal to zero. This gives the following optimization,
\begin{equation}
    \label{eq:app_20}
    \min_{\substack{\sigma\geq 0,\alpha\\\u \in \R^p\\}}~\max_{\substack{\beta\geq 0,\tau>0, \gamma\\\v\in \R^p}} ~ \beta \cdot \sqrt{c_{\kappa}\big(\alpha,\sigma\big)} - \frac{\sigma\tau}{2} -\frac{\sigma}{2\tau}\big(\delta\beta^2 + \gamma^2\kappa^2\big)+\frac{1}{p}\big[\psi(\u)+ \frac{\tau}{2\sigma}\lnorm{\u+\frac{\beta \sigma\sqrt{\delta}}{\tau}~\h+(\frac{\sigma\gamma}{\tau}-\alpha) \w^\star}^2\big].
\end{equation}
We now switch the order of $\min$ and $\max$ (similar to what we did earlier for $\tilde \w$) and perform the minimization with respect to $\u$. Using the definition of the Moreau envelope, we can write down this optimization in terms of the Moreau envelpe of the potential function. We have,
\begin{equation}
    \label{eq:app_21}
    M_{\psi(\cdot)}\big((\alpha-\frac{\sigma\gamma}{\tau}) \w^\star-\frac{\beta \sigma\sqrt{\delta}}{\tau}~\h,\frac{\sigma}{\tau}\big)~=~\min_{u\in \R^p}~\psi(\u)+ \frac{\tau}{2\sigma}\lnorm{\u+\frac{\beta \sigma\sqrt{\delta}}{\tau}~\h+(\frac{\sigma\gamma}{\tau}-\alpha) \w^\star}^2.
\end{equation}
Using the result of Lemma~\ref{lem:Lipschitz_moreau}, we have that the Moreau envelope is a Lipschitz function as $\psi(\cdot)$ is Lipschitz. Therefore, we can exploit the Gaussian concentration of Lipschitz functions (see Theorem~5.22 in~\cite{vershynin2018high}) which gives,
\begin{equation}
    \label{eq:app_22}
    \frac{1}{p}M_{\psi(\cdot)}\big((\alpha-\frac{\sigma\gamma}{\tau}) \w^\star-\frac{\beta \sigma\sqrt{\delta}}{\tau}~\h,\frac{\sigma}{\tau}\big)\overset{P}\longrightarrow ~\frac{1}{p}\Expect{\big[M_{\psi(\cdot)}\big((\alpha-\frac{\sigma\gamma}{\tau}) \w^\star-\frac{\beta \sigma\sqrt{\delta}}{\tau}~\h,\frac{\sigma}{\tau}\big)\big]}~,~~\text{as }p\rightarrow \infty.
\end{equation}
We now appeal to Lemma 9 in Appendix A of~\cite{thrampoulidis2018precise}, which allows us to replace the Moreau envelope with their expected value due to the convergence we are getting in~\eqref{eq:app_22}. Hence, by replacing the Expected value of the Moreau envelope function, we are getting the following optimization, to be analyzed in the next section.
\begin{equation}
    \label{eq:app_23}
    \min_{\substack{\sigma\geq 0,\alpha}}\max_{\substack{\gamma\\ \beta\geq 0,\tau> 0}} ~\frac{1}{p}\Expect{\big[M_{\psi(\cdot)}\big((\alpha-\frac{\sigma\gamma}{\tau}) \w^\star-\frac{\beta \sigma\sqrt{\delta}}{\tau}~\h,\frac{\sigma}{\tau}\big)\big]} + \beta \sqrt{c_{\kappa}\big(\alpha,\sigma\big)}-\frac{\sigma\tau}{2}-\frac{\sigma}{2\tau}(\delta\beta^2+\gamma^2\kappa^2).
\end{equation}

\subsection{Optimality condition of the auxiliary optimization}\label{sec:opt_cond}

In this section, we conclude the proof of the main result of the paper by showing that (when $\delta>\delta^*$) the optimizer to the scalar optimization~\eqref{eq:app_23} can be derived by solving the nonlinear system of equations~\eqref{eq:nonlinsys}.\\
Here, we investigate the optimality condition for the solution of the auxiliary optimization. In Section~\ref{sec:analyze_aux}, we simplified the (AO) and after some algebra we got the scalar optimization~\eqref{eq:app_23} with respect to five variables. Here, we would like to present the solution to this optimization. Let $C(\alpha,\sigma,\gamma,\beta,\tau)$ denote the objective function in the scalar optimization. In other words, the fucntion $C$ is defined as:
\begin{equation}
    \label{eq:app_24}
    C(\alpha,\sigma,\gamma,\beta,\tau) = \frac{1}{p}\Expect{\big[M_{\psi(\cdot)}\big((\alpha-\frac{\sigma\gamma}{\tau}) \w^\star+\frac{\beta \sigma\sqrt{\delta}}{\tau}~\h,\frac{\sigma}{\tau}\big)\big]} + \beta \sqrt{c_{\kappa}\big(\alpha,\sigma\big)}-\frac{\sigma\tau}{2}-\frac{\sigma}{2\tau}(\delta\beta^2+\gamma^2\kappa^2).
\end{equation}
The following lemma describes the behavior of the function $C$ with respect to its variables.
\begin{lem}
\label{lem:convex_concave}
The function $C:\R^5\rightarrow \R$ defined in~\eqref{eq:app_24} is (jointly) convex with respect to the variables $(\alpha, \sigma)$, and (jointly) concave with respect to the variables $(\gamma, \beta, \tau)$.
\end{lem}
The proof of this lemma is provided in Appendix~\ref{sec:app_pf_lem}. Using the result of Theorem~\ref{thm:phase_transition}, the objective function, $C$, will diverge when $\delta<\delta^*$. For $\delta>\delta^*$ Lemma~\ref{lem:convex_concave} states that the function $C$ is convex-concave. The following remark indicates that the optimal solution of the optimization problem does not happen at the boundry values.
\begin{rem}
We need to show that the optimal solution does not happen at the boundary, i.e., at $\beta = 0$, or $\sigma=0$. Taking the derivative with respect to $\beta$ at the objective function in~\eqref{eq:app_18}, we will have $\frac{\partial}{\partial \beta}|_{\beta=0}=\sqrt{c_{\kappa}(\alpha, \sigma)}>0 $. Therefore, the optimal $\beta$ is nonzero. It can also be seen in the same optimization program that when $\sigma=0$, $\beta$ can choose its value arbitrary large and the optimal value would be $+\infty$. Hence, the optimal $\sigma$ is also nonzero as we have a minimization w.r.t. $\sigma$.
\end{rem}
Let $(\bar \alpha, \bar \sigma, \bar \gamma, \bar \beta, \bar \tau)$ denote the solution to the optimization~\eqref{eq:app_23}. Since the objective function is smooth with respect to its variables and the optimal values do not coincide with the boundries, its solution must satisfy the first-order optimality condition, i.e., $\nabla C\big(\bar \alpha, \bar \sigma, \bar \gamma, \bar \beta, \bar \tau\big) = \mathbf 0_{5\times 1}$. We will show that this would simplify to our system of nonlinear equations~\eqref{eq:nonlinsys}. \\
We start by setting the derivative with respect to $\alpha$ to zero. We have,
\begin{equation}
    \label{eq:app_25}
    \frac{\partial C}{\partial \alpha} = 0 \Rightarrow \frac{1}{p}\Expect{\big[\frac{\partial}{\partial \alpha}M_{\psi(\cdot)}\big((\alpha-\frac{\sigma\gamma}{\tau}) \w^\star+\frac{\beta \sigma\sqrt{\delta}}{\tau}~\h,\frac{\sigma}{\tau}\big)\big]} + \frac{\beta}{2\sqrt{c_{\kappa}\big(\alpha,\sigma\big)}}\cdot \frac{\partial c_{\kappa}\big(\alpha,\sigma\big)}{\partial \alpha}~=0~,
\end{equation}
where we used the Leibniz integral rule to bring the derivative inside the expectation. Using the result of Lemma~\ref{lem:der_moreau}, we can write the following,
\begin{equation}
    \label{eq:app_26}
    \frac{1}{p}\Expect{\big[\frac{\partial}{\partial \alpha}M_{\psi(\cdot)}\big((\alpha-\frac{\sigma\gamma}{\tau}) \w^\star+\frac{\beta \sigma\sqrt{\delta}}{\tau}~\h,\frac{\sigma}{\tau}\big)\big]} = \frac{\kappa^2\alpha\tau}{\sigma} - \kappa^2\gamma -\frac{\tau}{p\sigma}\Expect\big[ {\w^\star}^T\Prox_{\frac{\sigma}{\tau}\psi(\cdot)}((\alpha-\frac{\sigma\gamma}{\tau}) \w^\star+\frac{\beta \sigma\sqrt{\delta}}{\tau}~\h\big)\big].
\end{equation}
Replacing~\eqref{eq:app_26} in~\eqref{eq:app_25} gives the following nonlinear equation,
\begin{equation}
    \label{eq:app_27}
    \frac{\tau}{p\sigma}\Expect\big[ {\w^\star}^T\Prox_{\frac{\sigma}{\tau}\psi(\cdot)}((\alpha-\frac{\sigma\gamma}{\tau}) \w^\star+\frac{\beta \sigma\sqrt{\delta}}{\tau}~\h\big)\big] = \frac{\kappa^2\alpha\tau}{\sigma} - \kappa^2\gamma + \frac{\beta}{2\sqrt{c_{\kappa}\big(\alpha,\sigma\big)}}\cdot \frac{\partial c_{\kappa}\big(\alpha,\sigma\big)}{\partial \alpha}
\end{equation}
Next, we find another optimality condition by setting the derivative with respect to $\beta$ to zero. We have,
\begin{equation}
    \label{eq:app_28}
    \frac{\partial C}{\partial \beta} = 0 \Rightarrow \frac{1}{p}\Expect{\big[\frac{\partial}{\partial \beta}M_{\psi(\cdot)}\big((\alpha-\frac{\sigma\gamma}{\tau}) \w^\star+\frac{\beta \sigma\sqrt{\delta}}{\tau}~\h,\frac{\sigma}{\tau}\big)\big]} + \sqrt{c_{\kappa}\big(\alpha,\sigma\big)} - \frac{\sigma\delta}{\tau}\beta ~=0~,
\end{equation}
Similar to~\eqref{eq:app_26}, we can compute the expected derivative of the Moreau envelope function by appealing to Lemma~\ref{lem:der_moreau},
\begin{equation}
    \label{eq:app_29}
    \frac{1}{p}\Expect{\big[\frac{\partial}{\partial \beta}M_{\psi(\cdot)}\big((\alpha-\frac{\sigma\gamma}{\tau}) \w^\star+\frac{\beta \sigma\sqrt{\delta}}{\tau}~\h,\frac{\sigma}{\tau}\big)\big]} = \frac{\beta\sigma\delta}{\tau}-\frac{\sqrt \delta}{p}\Expect\big[ \h^T\Prox_{\frac{\sigma}{\tau}\psi(\cdot)}((\alpha-\frac{\sigma\gamma}{\tau}) \w^\star+\frac{\beta \sigma\sqrt{\delta}}{\tau}~\h\big)\big].
\end{equation}
Replacing~\eqref{eq:app_29} in~\eqref{eq:app_28} will give the following nonlinear equation:
\begin{equation}
    \label{eq:app_30}
    \boxed{
    \frac{1}{p}\Expect\big[ \h^T\Prox_{\frac{\sigma}{\tau}\psi(\cdot)}((\alpha-\frac{\sigma\gamma}{\tau}) \w^\star+\frac{\beta \sigma\sqrt{\delta}}{\tau}~\h\big)\big] = \sqrt{\frac{c_{\kappa}\big(\alpha,\sigma\big)}{\delta}}~.
    }
    \tag{E2}
\end{equation}
Next, we compute the derivative with respect to $\gamma$ and set it to zero. We have,
\begin{equation}
    \label{eq:app_31}
    \frac{\partial C}{\partial \gamma} = 0 \Rightarrow \frac{1}{p}\Expect{\big[\frac{\partial}{\partial \gamma}M_{\psi(\cdot)}\big((\alpha-\frac{\sigma\gamma}{\tau}) \w^\star+\frac{\beta \sigma\sqrt{\delta}}{\tau}~\h,\frac{\sigma}{\tau}\big)\big]} -\frac{\sigma\kappa^2\gamma}{\tau} ~=0~,
\end{equation}
\begin{equation}
    \label{eq:app_32}
    \frac{1}{p}\Expect{\big[\frac{\partial}{\partial \gamma}M_{\psi(\cdot)}\big((\alpha-\frac{\sigma\gamma}{\tau}) \w^\star+\frac{\beta \sigma\sqrt{\delta}}{\tau}~\h,\frac{\sigma}{\tau}\big)\big]} = \frac{\sigma\kappa^2\gamma}{\tau}-\kappa^2\alpha-\frac{1}{p}\Expect\big[ {\w^\star}^T\Prox_{\frac{\sigma}{\tau}\psi(\cdot)}((\alpha-\frac{\sigma\gamma}{\tau}) \w^\star+\frac{\beta \sigma\sqrt{\delta}}{\tau}~\h\big)\big].
\end{equation}
Replacing~\eqref{eq:app_32} in~\eqref{eq:app_31} will give the following nonlinear equation:
\begin{equation}
    \label{eq:app_33}
    \boxed{
    \frac{1}{p}\Expect\big[ {\w^\star}^T\Prox_{\frac{\sigma}{\tau}\psi(\cdot)}((\alpha-\frac{\sigma\gamma}{\tau}) \w^\star+\frac{\beta \sigma\sqrt{\delta}}{\tau}~\h\big)\big] = \kappa^2\alpha~.
    }
    \tag{E1}
\end{equation}
Also, replacing~\eqref{eq:app_33} in the nonlinear equation~\eqref{eq:app_27} gives the following nonlinear equation:
\begin{equation}
    \label{eq:app_34}
    \boxed{
    \frac{\partial c_{\kappa}(\alpha,\sigma)}{\partial \alpha} = \frac{2\kappa^2\gamma}{\beta}\sqrt{c_{\kappa}(\alpha,\sigma)}~.
    }
    \tag{E4}
\end{equation}
Next, we take the derivative with respect to $\sigma$. We have:
\begin{equation}
    \label{eq:app_35}
    \frac{\partial C}{\partial \sigma} = 0\Rightarrow \frac{1}{p}\Expect{\big[\frac{\partial}{\partial \sigma}M_{\psi(\cdot)}\big((\alpha-\frac{\sigma\gamma}{\tau}) \w^\star+\frac{\beta \sigma\sqrt{\delta}}{\tau}~\h,\frac{\sigma}{\tau}\big)\big]}+\frac{\beta}{2\sqrt{c_{\kappa}(\alpha,\sigma)}} \frac{\partial c_{\kappa}(\alpha,\sigma)}{\partial \sigma}-\frac{\tau}{2}-\frac{1}{2\tau}(\delta\beta^2+\gamma^2\kappa^2) = 0,
\end{equation}
We use the result of the Lemma~\ref{lem:der_moreau} to compute the derivative of $M_{\psi}(\cdot, \cdot)$ with respect to $\sigma$. We have,
\begin{align}
    \label{eq:app_36}
    \frac{1}{p}\Expect{\big[\frac{\partial}{\partial \sigma}M_{\psi(\cdot)}\big((\alpha-\frac{\sigma\gamma}{\tau}) \w^\star+\frac{\beta \sigma\sqrt{\delta}}{\tau}~\h,\frac{\sigma}{\tau}\big)\big]} = \frac{1}{2\tau}\big({\gamma^2\kappa^2}+\delta\beta^2 + \frac{\alpha^2\kappa^2\tau^2}{\sigma^2} - \frac{\tau^2}{p\sigma^2}\Expect\lnorm{\Prox_{\frac{\sigma}{\tau}\psi(\cdot)}((\alpha-\frac{\sigma\gamma}{\tau}) \w^\star+\frac{\beta \sigma\sqrt{\delta}}{\tau}~\h\big)}^2\big)
\end{align}
Replacing this into~\eqref{eq:app_35} will give the following equation,
\begin{equation}
    \label{eq:app_37}
    \frac{1}{p}\Expect\lnorm{\Prox_{\frac{\sigma}{\tau}\psi(\cdot)}((\alpha-\frac{\sigma\gamma}{\tau}) \w^\star+\frac{\beta \sigma\sqrt{\delta}}{\tau}~\h\big)}^2=\alpha^2\kappa^2+\frac{\beta \sigma^2}{\tau\sqrt{c_{\kappa}(\alpha,\sigma)}} \frac{\partial c_{\kappa}(\alpha,\sigma)}{\partial \sigma}-\sigma^2
\end{equation}
Similarly, by taking the derivative with respect to $\tau$, we have:
\begin{equation}
    \label{eq:app_38}
    \boxed{
    \frac{1}{p}\Expect\lnorm{\Prox_{\frac{\sigma}{\tau}\psi(\cdot)}((\alpha-\frac{\sigma\gamma}{\tau}) \w^\star+\frac{\beta \sigma\sqrt{\delta}}{\tau}~\h\big)}^2 = \alpha^2\kappa^2 + \sigma^2}
    \tag{E3}
\end{equation}
We can now simplify~\eqref{eq:app_37} to get the following equation:
\begin{equation}
    \label{eq:app_39}
    \boxed{\frac{\partial c_{\kappa}(\alpha,\sigma)}{\partial \sigma} = \frac{2\tau\sqrt{c_{\kappa}(\alpha,\sigma)}}{\beta }}
    \tag{E5}
\end{equation}
Finally, we make a change of variable by replacing $\tau$ with $\frac{1}{\tau}$ in the equations~\eqref{eq:app_33},~\eqref{eq:app_30},~\eqref{eq:app_38},\eqref{eq:app_34}, and~\eqref{eq:app_39} will respectively give the desired equations in the system of nonlinear equations~\eqref{eq:nonlinsys} as the optimality condition on the solution of the optimization~\eqref{eq:app_23}. This concludes the proof.

\section{Proof of Theorem~\ref{thm:phase_transition}}
In this section we prove the result presented in Theorem~\ref{thm:phase_transition} which identifies the phase transition on the separability of the data. To this end, we exploit the result of Lemma~\ref{lem:define_aux} which associates the following optimization to the GMM optimization~\eqref{eq:app_main_opt}.
\begin{equation}
\begin{aligned}
    \label{eq:app_40}
    &&&\min_{\substack{\alpha\in \R\\\tilde\w \in \R^p\\\tilde \w\perp \w^\star}} ~\frac{1}{p}\psi(\alpha\w^\star+\tilde \w)\\ 
    &&&~\quad\qquad\text{s.t.}~~~\frac{1}{p}(\h^T\tilde \w)^2\geq n\cdot c_{\kappa}\big(\alpha, \frac{\lnorm{\tilde \w}}{\sqrt{p}}\big).
\end{aligned}
\end{equation}
We first show that, as $p,n\rightarrow \infty$, $\delta>\delta^*=\delta^*(\kappa)$ is the necessary and sufficient condition for the optimization program~\eqref{eq:app_40} to have a feasible solution. Define $\sigma:=\lnorm{\tilde \w}{\sqrt p}$, and write the following:
\begin{equation}
    \label{eq:app_41}
    \sup_{\substack{\alpha\in \R\\\tilde \w \perp \w^\star}}~\frac{1}{p}(\h^T\tilde \w)^2- n\cdot c_{\kappa}\big(\alpha, \frac{\lnorm{\tilde \w}}{\sqrt{p}}\big) =  \sup_{\sigma\geq 0, \alpha}~\sigma^2\cdot\lnorm{\P^{\perp}\h}^2 - n\cdot c_{\kappa}\big(\alpha, \sigma\big)~.
\end{equation}
Note that we used the fact that the $\P^\perp$ is the projection onto the hyperplane orthogonal to $\w^\star$. The supremum is achieved iff $\tilde \w$ chooses its direction to be the same as $\P^\perp \h$. The optimization program has a feasible point if and only if the optimal value in~\eqref{eq:app_41} be nonnegative. In other words, the necessary and sufficient condition on the separability of the data is:
\begin{equation}
    \label{eq:app_42}
    \exists~r\geq 0 ,~s~,~~\text{s.t. }~ r^2\cdot\lnorm{\P^{\perp}\h}^2 - n\cdot c_{\kappa}\big(s, r\big)\geq 0 ~\Longleftrightarrow~ \frac{1}{n}\lnorm{\P^\perp\h}^2\geq \delta^*= \underset{s,r\geq 0}\inf~\frac{c_{\kappa}(s,r)}{r^2}.
\end{equation}
Next we note that $\h$ has i.i.d. $\normal(0,1)$ entries, therefore, SLLN asserts that,
\begin{equation}
    \label{eq:app_43}
    \frac{1}{n}\lnorm{\P^\perp\h}\overset{\text~{a.s.}}\longrightarrow~ \frac{p-1}{n}.
\end{equation}
Therefore, as $n,p\rightarrow \infty$ with $\delta:=\frac{p}{n}$, the optimization program~\eqref{eq:app_40} is feasible if and only if $\delta>\delta^*$.As Lemma~\ref{lem:define_aux} states that the solution to the GMM optimization~\eqref{eq:app_main_opt} converges in probability to the solution of~\eqref{eq:app_40}. Therefore, $\delta>\delta^*$ indicates the phase transition for the existence of the GMM classifier.

We would also want to refer the interested reader to~\cite{cover1965geometrical} for an astute geometric/combinatorial perspective on the phase transition behavior in binary classification.

\section{Proof of Lemma~\ref{lem:convex_concave}}\label{sec:app_pf_lem}
Consider the objective function in the optimization program~\ref{eq:app_18}, i.e., 
\begin{equation}
f^{(p)}\big(\alpha,\sigma,\u;\gamma,\beta,\tau,\v\big)=\frac{1}{p}\psi(\u) + \beta \cdot \sqrt{c_{\kappa}\big(\alpha,\sigma\big)}+\frac{1}{p}\v^T\big(\u - \alpha\w^\star) - \frac{\sigma\tau}{2} -\frac{\sigma}{2\tau}\cdot\lnorm{\frac{\beta}{\sqrt n} \h-\frac{1}{\sqrt p}\v+\frac{\gamma}{\sqrt p} \w^\star}^2.
\end{equation}
First, we would like to show that $f^{(p)}$ is jointly convex with respect to $\alpha$, $\sigma$ and $\u$. From Lemma~\eqref{lem:c_kappa_convex}, we know that $\sqrt{c_{\kappa}\big(\alpha, \sigma\big)}$ is jointly convex with respect to $\alpha$ and $\sigma$. The function $\psi(\cdot)$ is also convex and the remaining terms are all linear with respect to these three variables. Hence, $f^{(p)}$ is convex with respect to $\u$, $\alpha$ and $\sigma$. \\
Next, we show that this function is jointly concave with respect  to the remaining variables. We note that the function $\lnorm{\frac{\beta}{\sqrt n} \h-\frac{1}{\sqrt p}\v+\frac{\gamma}{\sqrt p} \w^\star}^2$ is convex with respect to variables $\v$, $\gamma$, and $\beta$. The perspective of this function $\frac{1}{\tau}\lnorm{\frac{\beta}{\sqrt n} \h-\frac{1}{\sqrt p}\v+\frac{\gamma}{\sqrt p} \w^\star}^2$ is (jointly) convex with respect to $(\gamma,\beta,\tau, \v)$. Therefore, $f^{(p)}$ is jointly convex with respect to these variables as the remaining terms are affine with respect to $(\gamma,\beta,\tau, \v)$. Next, we define the function $C^{(p)}$ by maximizing $f^{(p)}$ with respect to $\v$, i.e.,
\begin{equation}
    C^{(p)}\big(\alpha,\sigma,\u;\gamma,\beta,\tau \big)= \max_{\v \in \R^p}f^{(p)}\big(\alpha,\sigma,\u;\gamma,\beta,\tau,\v\big)
\end{equation}
This function is also jointly convex-concave, since it is a point-wise maximum of concave function with repect to $\v$. The result is the consequence of the fact that $C^{(p)}$ converges to $C$, i.e., \begin{equation}
  C^{(p)}\big(\alpha,\sigma,\u;\gamma,\beta,\tau \big)\overset{P}\rightarrow C(\alpha,\sigma, \gamma,\beta,\tau)~,~~\text{as }~p\rightarrow \infty .  
\end{equation}  

\section{GMM for Various Structures}
In this section, we provide some technical details on how to characterize the performance of the classifiers introduced in Section~\ref{sec:structured}. For each of the three classifiers, depending on the distribution of the underlying parameter ($\w^\star$) we simplify the nonlinear system~\eqref{eq:nonlinsys} by explicitly evaluating the expected values. 
\subsection{Max-margin classifier ($\ell_2$-GMM)}
As mentioned earlier, when $\psi(\cdot)=\frac{1}{2}\lnorm{\cdot}_2^2$, the GMM classifier will become the well-known max-margin classifier. In this case, we can find the following closed-form for the proximal operator:
\begin{equation}
    \label{eq:app_44}
    \text{Prox}_{\frac{t}{2}\lnorm{\cdot}^2} (\v) = \frac{1}{1+t} \v. 
\end{equation}
Therefore, the expectations in the nonlinear system~\eqref{eq:nonlinsys} can be computed explicitly as follows:
\begin{equation}
    \label{eq:app_45}
    \begin{cases}
    \begin{aligned}
    &&\frac{1}{p}\Expect\big[ {\w^\star}^T\Prox_{\frac{\sigma\tau}{2}\lnorm{\cdot}^2}\big((\alpha-\sigma\tau\gamma)\w^\star+\beta \sigma\tau\sqrt\delta \h\big)\big] &= \frac{\kappa^2\big(\alpha-\sigma\tau\gamma\big)}{1+\sigma\tau},\\
    &&\frac{1}{p}\Expect\big[ \h^T\Prox_{\sigma\tau\psi(\cdot)}\big((\alpha-\sigma\tau\gamma)\w^\star+\beta \sigma\tau\sqrt\delta \h\big)\big] &= \frac{\beta\sigma\tau\sqrt{\delta}}{1+\sigma\tau},\\
    &&\frac{1}{p}\Expect\lnorm{\Prox_{\sigma\tau\psi(\cdot)}\big((\alpha-\sigma\tau\gamma)\w^\star+\beta \sigma\tau\sqrt\delta \h\big)}^2 &= \frac{\kappa^2\big(\alpha-\sigma\tau\gamma\big)^2+\beta^2\sigma^2\tau^2\delta}{\big(1+\sigma\tau\big)^2}.
    \end{aligned}
    \end{cases}
\end{equation}
Replacing these evaluations into the first three equations in the nonlinear system~\eqref{eq:nonlinsys}, will explicitly give two of the variables in terms of the other three variables. More specifically, we get $\gamma = -\alpha$ from the first equation, and $\beta = \frac{1+\sigma\tau}{\tau\sqrt \delta}$ from the thrid equation in the nonlinar system~\eqref{eq:nonlinsys}. Hence, the nonlinear system would reduce to solving the following system of $3$ nonlinear equations with $3$ unknowns:
\begin{equation}
\label{eq:app_nonlinsys_l2}
    \begin{cases}
    \begin{aligned}
    &&\sqrt{c_{\kappa}(\alpha,\sigma)} &= \sigma\sqrt{\delta},\\
        &&\frac{\partial c_{\kappa}(\alpha,\sigma)}{\partial \alpha} &= \frac{-2\kappa^2\alpha\tau\sigma\delta}{1+\sigma\tau},\\
    &&\frac{\partial c_{\kappa}(\alpha,\sigma)}{\partial \sigma} &= \frac{2\sigma\delta}{1+\sigma\tau}.
    \end{aligned}
    \end{cases}
\end{equation}
\subsection{Sparse classifier ($\ell_1$-GMM)}\label{sec:app_prox_l1}
The second choice for the potential function is $\psi(\cdot) = \lnorm{\cdot}_1$, which is used to promote sparsity in the underlying parameter. Here, we assume that the entries of the underlying parameter are generated independently from the distribution $\Pi_s$ introduced in~\eqref{eq:sparse_dist}, where $s\in(0,1)$ denotes the sparsity factor which indicates the probability of an entry being nonzero. The nonzero entries have Gaussian distribution with variance $\kappa^2/s$. The proximal operator for $\ell_1$ norm can be computed explicitly as,
\begin{equation}
    \label{eq:app_prox_l1}
    \text{Prox}_{t\lnorm{\cdot}_1}(\u) = \eta(\u,t),
\end{equation}
where $\eta(x,t) = \frac{x}{|x|}\big(|x|-t\big)_+$ is the soft tresholding function that has been applied entrywise. The expectations that appear in the first three equations in the nonlinear system~\eqref{eq:nonlinsys} can be presented as follows:
\begin{equation}
    \label{eq:app_46}
    \begin{cases}
    \begin{aligned}
    &&\frac{1}{p}\Expect\big[ {\w^\star}^T\Prox_{\frac{\sigma\tau}{2}\lnorm{\cdot}^2}\big((\alpha-\sigma\tau\gamma)\w^\star+\beta \sigma\tau\sqrt\delta \h\big)\big] &=~2\kappa^2\cdot Q(t_1)\cdot\big(\alpha-\sigma\tau\gamma\big) ,\\
    &&\frac{1}{p}\Expect\big[ \h^T\Prox_{\sigma\tau\psi(\cdot)}\big((\alpha-\sigma\tau\gamma)\w^\star+\beta \sigma\tau\sqrt\delta \h\big)\big] &=  \big[2sQ(t_1) +2(1-s)Q(t_2) \big]\cdot \beta\sigma\tau\sqrt{\delta},\\
    &&\frac{1}{p}\Expect\lnorm{\Prox_{\sigma\tau\psi(\cdot)}\big((\alpha-\sigma\tau\gamma)\w^\star+\beta \sigma\tau\sqrt\delta \h\big)}^2 &= 2\sigma^2\tau^2\big(\frac{s}{t_1^2}\cdot\chi(t_1)+\frac{1-s}{t_2^2}\cdot\chi(t_2)\big), 
    \end{aligned}
    \end{cases}
\end{equation}
where $t_1$ and $t_2$ are defined as,
\begin{equation}
    \label{eq:app_47}
    t_1 = \frac{\sigma\tau}{\sqrt{\frac{\kappa^2}{s}(\alpha-\sigma\tau\gamma)^2 + \beta^2\sigma^2\tau^2\delta}}~,~~ t_2 =\frac{1}{\beta\sqrt{\delta}},
\end{equation}
and the function $\chi:\R\rightarrow \R_+$ is defined as:
\begin{equation}
    \label{eq:app_48}
    \chi(t) = \Expect[\big(Z-t\big)_+^2] = Q(t)\big(1+t^2\big)-t\phi(t)~,
\end{equation}
where the random variable $Z$ in the above expectation have standard normal distribution, and $\phi(x) = \frac{1}{\sqrt{2\pi}}\exp(-x^2/2)$ denotes the density of the standard normal distribution. Replacing the computed expectations in~\eqref{eq:app_46} in the nonlinear system~\eqref{eq:nonlinsys} gives the sparse nonlinear system presented in~\eqref{eq:nonlinsys_l1}.\\
It is worth mentioning that the sparse nonlinear system~\eqref{eq:nonlinsys_l1} can be solved efficiently via iterative numerical methods. A main advantage of the sparse nonlinear system is that it has be presented in terms of the $Q(\cdot)$ function which can be computed quickly in most numerical softwares (e.g. MATLAB). For our numerical simulations in Section~\ref{sec:num_sim} we used an accelerated fixed-point iterative method to find the solution of the nonlinear system. 

\subsection{Binary classifier ($\ell_{\infty}$-GMM)}\label{sec:app_bin_classifier}
The third and last choice of the potential function is the $\ell_\infty$ norm. In this case the potential function is defined as $\psi(\cdot) = p\lnorm{\cdot}_{\infty}$\footnote{The multiplication by the dimension, $p$, is necessary to ensure that all the terms in the optimization have constant ($\mathcal O(1)$) order.}. The following lemma determines how to compute  the proximal operator in this case.
\begin{lem}
\label{lem:prox_l_infty}
Let $\u\in\R^p$ have i.i.d. entries from a distribution $\Pi$. Then, for $t>0$, we have:
\begin{equation}
    \label{eq:app_49}
    \text{Prox}_{tp\lnorm{\cdot}_\infty}\big(\u\big) = \u - \text{Prox}_{\lambda \lnorm{\cdot}_1}(\u),
\end{equation}
where $\lambda$ is defined as,
\begin{enumerate}
    \item for $t\leq \Expect{|W|}$, $\lambda$ is the unique solution of $\Expect{\big[\big(|W|-\lambda\big)_+\big]} = t$.
    \item for $t\geq \Expect{|W|}$, then $\lambda=0$.
\end{enumerate}
\end{lem}
In the following subsections, we use the result of Lemma~\ref{lem:prox_l_infty} to compute the proximal operator for two different models (i.e., two different distributions on the entries of $\w^\star$.)
\subsubsection{$\ell_{\infty}$-GMM with sparse parameter}
Here, we consider the case where the entries of $\w^\star$ are drawn independently from the distribution $\Pi_s$ defined in~\eqref{eq:sparse_dist}. Note that when we set $s$ to $1$ this distribution will be the same as i.i.d. Gaussian entries. Hence, the result in this section can be applied to the non-sparse setting (when the underlying parameter has i.i.d. Gaussian entries.) \\
Using the result of Lemma~\ref{lem:prox_l_infty}, in this case the proximal operator can be computed as follows,
\begin{equation}
    \label{eq:app_50}
    \text{Prox}_{\sigma\tau p \lnorm{\cdot}_{\infty}}\big((\alpha-\sigma\tau\gamma)\w^\star+\beta \sigma\tau\sqrt\delta \h\big) = (\alpha-\sigma\tau\gamma)\w^\star+\beta \sigma\tau\sqrt\delta \h - \text{Prox}_{\lambda\sigma\tau \lnorm{\cdot}_{1}}\big((\alpha-\sigma\tau\gamma)\w^\star+\beta \sigma\tau\sqrt\delta \h\big).
\end{equation}
where $\lambda$ is defined in terms of the proxies $t_1$ and $t_2$ (defined in~\eqref{eq:app_47}):
\begin{enumerate}
\item If $\frac{s}{t_1}+\frac{1-s}{t_2}>\sqrt{\frac{\pi}{2}}$, then $\lambda$ is the unique solution of the following nonlinear equation:
\begin{equation}
    \label{eq:app_51}
    2s\cdot\big[\frac{1}{t_1}\phi(\lambda t_1) - \lambda Q(\lambda t_1)\big] + 2(1-s)\big[\frac{1}{t_2}\phi(\lambda t_2) - \lambda Q(\lambda t_2)\big] =1~.
\end{equation}
\item If $\frac{s}{t_1}+\frac{1-s}{t_2}\leq\sqrt{\frac{\pi}{2}}$, then $\lambda=0$.
\end{enumerate}
Therefore, after finding the value of $\lambda$ by solving equation~\eqref{eq:app_51}, the proximal operator which appears in the first three equations of the nonlinear system~\eqref{eq:nonlinsys} can be written explicitly in terms of the proximal operator of the $\ell_1$ norm which was illustrated in Section~\ref{sec:app_prox_l1}. Also, similar to the case of $\ell_1$-GMM, the expectations are written in terms of the functions $Q(\cdot)$, and $\phi(\cdot)$. Therefore, the solution to the nonlinear system can be found efficiently using numerical solvers.
\subsubsection{$\ell_{\infty}$-GMM with binary parameter}
Here, we consider the case where $\w^\star$ has i.i.d. entries  with distribution $\Pi = \kappa\cdot\Rad(\frac{1}{2})$. To simplify our presentation, we define the following proxy:
$$t_3 =  \big(\frac{\alpha}{\sigma\tau}-\gamma\big)\cdot\kappa~.$$
Using the result of Lemma~\ref{lem:prox_l_infty}, in this case the proximal operator can be computed as follows,
\begin{equation}
    \label{eq:app_52}
    \text{Prox}_{\sigma\tau p \lnorm{\cdot}_{\infty}}\big((\alpha-\sigma\tau\gamma)\w^\star+\beta \sigma\tau\sqrt\delta \h\big) = (\alpha-\sigma\tau\gamma)\w^\star+\beta \sigma\tau\sqrt\delta \h - \text{Prox}_{\lambda\sigma\tau \lnorm{\cdot}_{1}}\big((\alpha-\sigma\tau\gamma)\w^\star+\beta \sigma\tau\sqrt\delta \h\big).
\end{equation}
where $\lambda$ is defined as:
\begin{enumerate}
    \item When $\beta\sqrt{\delta}\cdot\phi(-\frac{t_3}{\beta\sqrt{\delta}})+t_3\cdot Q(-\frac{t_3}{\beta\sqrt{\delta}})>\frac{1}{2}$, $\lambda$ is defined as the unique solution of the following equations:
    \begin{equation}
    \label{eq:app_53}
    \beta\sqrt{\delta}\cdot\phi(\frac{\lambda-t_3}{\beta\sqrt{\delta}})+(t_3-\lambda)\cdot Q(\frac{\lambda-t_3}{\beta\sqrt{\delta}})=\frac{1}{2}
    \end{equation}
    \item Otherwise, $\lambda=0$.
\end{enumerate}
Hence, $\lambda$ can be computed by solving the  equation~\eqref{eq:app_53}, and consequently the proximal operator which appears in the first three equations of the nonlinear system~\eqref{eq:nonlinsys} can be written explicitly in terms of the proximal operator of the $\ell_1$ norm which was illustrated in Section~\ref{sec:app_prox_l1}. 
\section{Mathematical Tools}
\subsection{Some useful lemmas}
We gathered here some useful mathematical lemmas that are used in the proof of our main results. The following two lemmas are borrowed from~\cite{salehi2019impact}, and will be used to handle the Moreau envelope of the potential function. We refer the interested reader to~\cite{jourani2014differential} for a detailed study of the properties of the Moreau envelope functions.
\begin{lem}
\label{lem:der_moreau}
Consider the Moreau envelope of the function $\Phi:\R^d\rightarrow \R$, defined as:
\begin{equation}
    \label{eq:app_1}
    M_{\Phi(\cdot)}(\v, t) = \min_{\x\in \R^d}~\Phi(\x) + \frac{1}{2t}\lnorm{\x- \v}^2.
\end{equation}
The derivatives of the function $M_{\Phi(\cdot)}(\cdot, \cdot)$ can be computed as follows:
\begin{equation}
    \label{eq:app_2}
    \frac{\partial M_{\Phi(\cdot)}}{\partial \v} = \frac{1}{t}\big(\v  -\text{Prox}_{t\Phi(\cdot)}(\v)\big)~~~,~~~\frac{\partial M_{\Phi(\cdot)}}{\partial t} =-\frac{1}{2t^2}\lnorm{\v  -\text{Prox}_{t\Phi(\cdot)}(\v)}^2,
\end{equation}
where $\text{Prox}_{t\Phi(\cdot)}(\v)$ is the unique solution of the optimization~\eqref{eq:app_1}.
\end{lem}
\begin{lem}
Let $\Phi:\R^d\rightarrow \R$ be an invariantly separable funciton such that for all $\x\in \R^d$, $\Phi(\x)= \sum_{i=1}^{d}\phi (x_i)$ where $\phi(\cdot)$ is a real-valued function. Then, for all $(\v,t)\in \R^d\times \R_+$,
\begin{equation}
    \label{eq:app_3}
    M_{\Phi(\cdot)}(\v,t) = \sum_{i=1}^{d} M_{\phi(\cdot)}(v_i,t)~~,~~\text{and}~~~\text{Prox}_{t\Phi(\cdot)}(\v)=\begin{bmatrix}\text{Prox}_{t\phi(\cdot)}\big(v_1\big)\\\text{Prox}_{t\phi(\cdot)}\big(v_2\big)\\ \vdots \\ \text{Prox}_{t\phi(\cdot)}\big(v_d\big)\end{bmatrix}.
\end{equation}
\end{lem}
In the next lemma, we show that the Moreau envelope of a Lipschitz function is itself a Lipschitz function.
\begin{lem}
\label{lem:Lipschitz_moreau}
Let $\Phi:\R^d\rightarrow \R$ be an $L$-Lipschitz function. Then,  $M_{\Phi(\cdot)}(\cdot,t)$ is a $2L$-Lipschitz function, i.e., for all $\u,\v\in \R^d$,
\begin{equation}
    |M_{\Phi(\cdot)}(\u,t) -M_{\Phi(\cdot)}(\v,t)|\leq 2L\lnorm{\u-\v}.
\end{equation}
\end{lem}
\begin{proof}
In order to show this result, we need to find an upper bound on the dericative of the Moreau envelope. For all $\v\in \R^d$ we have,
\begin{equation}
    \begin{aligned}
    &&L\lnorm{\v  -\text{Prox}_{t\Phi(\cdot)}(\v)}&\geq \Phi\big(\v\big) - \Phi\big(\text{Prox}_{t\Phi(\cdot)}(\v)\big)\\
    &&&\geq \frac{1}{2t}\lnorm{\v  -\text{Prox}_{t\Phi(\cdot)}(\v)}^2~,
    \end{aligned}
\end{equation}
where the first inequality is due to the $L$-Lipschitzness of the function $\Phi(\cdot)$, and the second inequality is derived from the fact that $\text{Prox}_{t\Phi(\cdot)}(\v)$ is the solution to the optimization~\eqref{eq:app_1}. This gives the following bound on the distance of the proximal operator to the underlying vector. 
\begin{equation}
    \lnorm{\v  -\text{Prox}_{t\Phi(\cdot)}(\v)}\leq 2tL.
\end{equation}
We can now bound the derivative $\frac{\partial M_{\Phi(\cdot)}}{\v}$ as follows,
\begin{equation}
    \lnorm{\frac{\partial M_{\Phi(\cdot)}}{\partial \v}} = \frac{1}{t}\lnorm{\big(\v  -\text{Prox}_{t\Phi(\cdot)}(\v)\big)}\leq 2L~,~\forall \v \in \R^d.
\end{equation}
This concludes the proof.
\end{proof}
The following lemma provides some information on the summary functional $c_{\kappa}(\cdot,\cdot)$, which will be used later in Section~\ref{sec:opt_cond} to find the optimality condition for the solution of a scalar optimization.
\begin{lem}
\label{lem:c_kappa_convex}
The function $f(s,r) := \sqrt{c_{\kappa}(s,r)}$ is (jointly) convex in (s,r).
\end{lem}
\begin{proof}
First, note that for $\x\in\R^n$, the function $\x \mapsto  \lnorm{(\x)_+}$ is a convex function as it can be written as a supremum of convex(linear) functions.
\begin{equation}
    \lnorm{(\x)_+} = \sup_{\substack{\u\in \R_+^n\\\lnorm{\u}\leq 1}}\u^T\x.
\end{equation}
For $n\in \mathbb N$ define the function $f_\kappa^{(n)}(s,r)$ as:
\begin{equation}
    f_\kappa^{(n)}(s,r) = \frac{1}{\sqrt{n}}\lnorm{\big(\mathbf 1_{n}-s\kappa\h\y+r\g\y\big)_+},
\end{equation}
where $\frac{1}{\sqrt{n}}$ denote the all-one vector, $\h,\g\in\R^n$ have i.i.d. $\normal(0,1)$ entries and $Y\sim \Rad\big(\rho(\kappa\h)\big)$. It is readily seen that $f_\kappa^{(n)}(s,r)$ is jointly convex with respect to $s$ and $r$ as it is a combination of a convex function and a linear function. Using the LLN, we also have that,
\begin{equation}
    f_\kappa^{(n)}(s,r) \overset{P}\longrightarrow f(s,r) =  \sqrt{c_{\kappa}(s,r)}.
\end{equation}
Therefore, $f(s,r)$ is a convex function as it is a point-wise limit of convex functions.
\end{proof}

\subsection{Convex Gaussian min-max theorem (CGMT)}
\label{sec:CGMT}
Our analysis of the generalizaed margin maximizer optimization is based on the recently developed convex Gaussian min-max theorem (CGMT). As mentioned  earlier in Section~\ref{sec:intro}, the CGMT framework associates with a Primary Optimization (PO), a nearly-separable Auxiliary Optimization (AO), from which various properties of the primary optimization, such as the phase transition, can be investigated.\\
Let the (PO) and the (AO) problems be  defined respectively as follows:
\begin{subequations}\label{eq:POAO}
\begin{align}
\label{eq:PO_gen}
\Phi(\mathbf G)&:= \min_{\mathbf w\in\mathcal S_{\mathbf w}}~\max_{\mathbf u\in \mathcal S_{\mathbf u}}~ \mathbf u^T\mathbf G\mathbf w + f(\mathbf u, \mathbf w),\tag{PO}\\
\label{eq:AO_gen}
\phi(\mathbf g,\mathbf h)&:= \min_{\mathbf w\in\mathcal S_{\mathbf w}}~\max_{\mathbf u\in \mathcal S_{\mathbf u}}~ ||\mathbf w||\mathbf g^T\mathbf u + ||\mathbf u||\mathbf h^T\mathbf w + f(\mathbf u, \mathbf w),\tag{AO}
\end{align}
\end{subequations}
where $\mathbf G\in\mathbb R^{m\times n}, \mathbf g\in\mathbb R^m, \mathbf h\in\mathbb R^n$, $\mathcal S_{\mathbf w}\subset\mathbb R^n,\mathcal S_{\mathbf u}\subset\mathbb R^m$ and $f:\mathbb R^n\times\mathbb R^m\rightarrow\mathbb R$. Denote by $\mathbf w_\Phi:=\mathbf w_\Phi(\mathbf G)$ and $\mathbf w_\phi:=\mathbf w_\phi(\mathbf g,\mathbf h)$ any optimal minimizers in \eqref{eq:PO_gen} and \eqref{eq:AO_gen}, respectively.

\begin{theorem}[CGMT]\cite{thrampoulidis2016recovering}
In~\eqref{eq:PO_gen}, let $\mathcal S_\mathbf w$, $\mathcal S_{\mathbf u}$, be convex and compact sets, and assume $f(\cdot,\cdot)$ is convex-concave on $\mathcal S_{\mathbf w}\times\mathcal S_{\mathbf u}$. Also assume that $\mathbf G,~\mathbf g,$ and $\mathbf h$ all have entries i.i.d. standard normal. The following statements are true (the probabilities are taken with respect to the randomness in $\mathbf G$, $\mathbf g$, and $\mathbf h$.),
\begin{enumerate}
    \item for all $\mu\in \mathbb R$, and $t>0$,
    \begin{equation}
        \mathbb P(|\Phi(\mathbf G)-\mu|>t)\leq 2\mathbb P(|\phi(\mathbf g, \mathbf h)-\mu|\geq t)~.
    \end{equation}
    \item Let $\mathcal S$ be an arbitrary open subset of $\mathcal S_{\mathbf w}$ and $\mathcal S^c:=\mathcal S_{\mathbf w}/\mathcal S$. Denote $\Phi_{\mathcal S^c}(\mathbf G)$ and $\phi_{\mathcal S^c}(\mathbf g, \mathbf h)$ be the optimal costs of the optimizations in~\eqref{eq:PO_gen}, and ~\eqref{eq:AO_gen}, respectively, when the minimization over $\mathbf w$ is now constrained over $\mathbf w \in \mathcal S^c$. If there exists constants $\bar \phi$, ${\bar \phi}_{\mathcal S^c}$, and $\eta>0$ such that,
    \begin{itemize}
        \item ${\bar \phi}_{\mathcal S^c}\geq \bar \phi +3\eta$~,\\
        \item $\phi(\mathbf g, \mathbf h) < \bar \phi +\eta$, with probability at least $1-p$~,\\
        \item $\phi_{\mathcal S^c}(\mathbf g, \mathbf h)>{\bar \phi}_{\mathcal S^c}-\eta$, with probability at least $1-p$~,
    \end{itemize}
    then,
    $\mathbb P(\mathbf w_{\Phi}(\mathbf G)\in \mathcal S)\geq 1-4p$~.
\end{enumerate}
\label{thm:CGMT}
\end{theorem}
In the asymptotic regime, we often appeal to the following corollary which is an immediate consequence of Part 2 in Theorem~\ref{thm:CGMT}.
\begin{cor}[Asymptotic CGMT]~\cite{thrampoulidis2016recovering} using the same notations and assumptions as in Theorem~\ref{thm:CGMT}, suppose there exists constants $\bar \phi<{\bar \phi}_{\mathcal S^c}$ such that $\phi(\mathbf g, \mathbf h)\overset{P}\longrightarrow \bar \phi$, and $\phi_{\mathcal S^c}(\mathbf g, \mathbf h)\overset{P}\longrightarrow {\bar \phi}_{\mathcal S^c}$. Then, 
\begin{equation}
    \lim_{n\rightarrow \infty}\mathbb P(\mathbf w_{\Phi}(\mathbf G)\in \mathcal S)=1~. 
\end{equation}
\label{cor:CGMT}
\end{cor}
For further reading on the subject, please refer to~\cite{ thrampoulidis2015regularized, thrampoulidis2018precise}.

 

\end{document}